\let\accentvec\vec 
\documentclass[runningheads,a4paper,envcountsame,oribibl]{llncs}

\usepackage{calculi+algebra}

\begin{document}

\title{Algebraic Properties of Qualitative Spatio-Temporal Calculi\thanks{The final publication is available at \url{http://link.springer.com}; this version contains supplementary material}}

\author{Frank Dylla\inst{1}, Till Mossakowski\inst{1,2} \and Thomas Schneider\inst{3} \and Diedrich Wolter\inst{1}}
\institute{%
  Collaborative Research Center on Spatial Cognition (SFB/TR 8), Univ. of Bremen\quad \url{{dylla,till,dwolter}@informatik.uni-bremen.de}%
  \and
  DFKI GmbH, Bremen\quad \url{Till.Mossakowski@dfki.de}
  \and
  Department of Mathematics and Computer Science, University of Bremen\quad \url{tschneider@informatik.uni-bremen.de}
}

\maketitle

\begin{abstract}
 Qualitative spatial and temporal reasoning is based on so-called qualitative calculi.
 Algebraic properties of these calculi have several implications on reasoning algorithms.
 But what exactly is a qualitative calculus? And to which extent do the qualitative
 calculi proposed meet these demands?
 The literature provides various answers to the first question but only few facts about
 the second. 
 In this paper we identify the  minimal requirements to binary spatio-temporal calculi 
 and we discuss the relevance of the according axioms for representation and reasoning.
 We also analyze existing qualitative calculi and provide a classification involving 
 different notions of relation algebra.
 
\end{abstract}

\section{Introduction}
Qualitative spatial and temporal reasoning is a sub-field of knowledge representation involved with representations of spatial and temporal domains that are based on finite sets of so-called qualitative relations. 
Qualitative relations serve to explicate knowledge relevant for a task at hand while at the same time they abstract from irrelevant knowledge. 
Often, these relations aim to relate to cognitive concepts.
Qualitative spatial and temporal reasoning thus link cognitive approaches to knowledge representation with formal methods. 
Computationally, qualitative spatial and temporal reasoning is largely involved with constraint satisfaction problems over infinite domains where qualitative relations serve as constraints.
Typical domains include, on the temporal side, points and intervals and, on the spatial side, regions or oriented points in the Euclidean plane. 
In the past decades, a vast number of qualitative representations have been developed that are commonly referred to as qualitative calculi (see \cite{Ligozat11} for a recent overview).
Yet the literature provides us with several definitions of what a qualitative calculus exactly is.
Nebel and Scivos \cite{DBLP:journals/ki/NebelS02} have introduced the
rather general and weak notion of a \emph{constraint algebra}, which is
a set of relations closed under converse, finite intersection, and
composition. Ligozat and Renz \cite{LigozatR04} focus on so-called
non-associative algebras, which are relation algebras without
associativity axioms, and which have a much richer structure.  Both
approaches assume that the converse operation is \emph{strong}, which
is not the case for calculi like the 
Cardinal Direction (Relations) Calculus (CDR) \cite{DBLP:journals/ai/SkiadopoulosK05}
or its recently introduced rectangular variant (RDR) \cite{NavarreteEtAl13}.  

The goal of this paper is to relate the existing definitions and to identify the essential representation-theoretic properties that characterize a qualitative calculus.
It is achieved by the following contributions:
\begin{itemize}
  \item
    We propose a definition of a qualitative calculus that includes existing spatio-temporal calculi
    by weakening the conditions usually imposed on the converse and composition relation
    (Section 2).
  \item
    We generalize the notions of constraint algebra and non-associative algebra
    to cover calculi with weak converse (Section 3.1).
  \item
    We discuss the role of algebraic properties of calculi
    for spatial reasoning,
    especially in connection with general-purpose reasoning tools like SparQ
    \cite{sparqmanual,WFW+06} and GQR \cite{r4:aaai08ws} (Section 3.2).
  \item
    We experimentally evaluate the algebraic properties of calculi
    and derive a reasoning procedure that is sensitive to these properties
    (Section 4).
  \item
    We examine information preservation properties of calculi during reasoning,
    i.e., how general relations evolve after several compositions (Section 5).
\end{itemize}

\section{Qualitative Representations}

In this section, we formulate minimal requirements to a qualitative calculus,
discuss their relevance to spatio-temporal representation and reasoning,
and list existing calculi.
We restrict ourselves to calculi with binary relations
because we want to examine their algebraic properties
using the notion of a relation algebra,
which is best understood for binary relations.

\subsection{Requirements to Qualitative Representations}
\label{sec:requirements}

We start with minimal requirements used in the literature.
Let us first fix some notation. Let $r,s,t$ range over binary relations over a non-empty universe $\Univ$, i.e., $r \subseteq \Univ \times \Univ$.
We use $\cup$, $\cap$, $\bar{~}$, $\breve{~}$ and $\circ$ to denote the union, intersection, complement, converse, and composition of relations,
as well as the identity and universal relations $\id = \{(u,u) \mid u \in \Univ\}$ and $\univ = \Univ \times \Univ$.
A relation $r \subseteq \Univ \times \Univ$ is called \emph{serial}
if, for every $u \in \Univ$, there is some $v \in \Univ$ such that $(u,v) \in r$.

Ligozat and Renz \cite{LigozatR04} note that most spatial and temporal calculi
are based on a set of JEPD (jointly exhaustive and pairwise disjoint) relations.
The following definition is standard in the QSR literature \cite{LigozatR04,Cohn:2008vn}.
\begin{defi}
  \label{def:JEPD+partition_schemes}
  Let $\Univ$ be a non-empty universe and $\URel$ a set of non-empty binary relations over $\Univ$.
  $\URel$ is called a set of \emph{JEPD relations} over $\Univ$ if
  the relations in $\URel$ are pairwise disjoint and $\Univ \times \Univ = \bigcup_{r \in \URel} r$.
  \par
  An \emph{abstract partition scheme} is a pair $(\Univ,\URel)$
  where $\URel$ is a set of \emph{JEPD relations} over $\Univ$.
  $(\Univ,\URel)$ is called a \emph{partition scheme} \cite{LigozatR04} if
  $\URel$ contains the identity relation $\id$
  and, for every $r \in \URel$, there is some $s \in \URel$ such that $r\breve{~} = s$.
\end{defi}

The universe \Univ represents the set of all spatial or temporal entities,
and \URel being a set of JEPD relations
ensures that each two entities are in exactly one relation from \URel.
Incomplete information about two entities is modeled by taking the union of base relations,
with the universal relation (the union of all base relations) representing that no information is available.
Disjointness of the base relations ensures that there is a unique way to represent
an arbitrary relation, and exhaustiveness ensures that the empty relation
can never occur in a consistent set of constraints (which are defined in Section \ref{sec:reasoning}).

Ligozat and Renz \cite{LigozatR04} base their definition of a
qualitative calculus on the notion of a partition scheme. This excludes
calculi like CDR and RDR which do not have strong converses. Hence,
we take a more general approach based on the notion of an abstract partition
scheme.  This accommodates existing calculi with these weaker
properties: some existing spatio-temporal representations do not
require an identity relation, and some representations are
deliberately kept coarse and thus do not guarantee that the converse
of a base relation is again a (base) relation.  Furthermore, the
computation of the converse operation may be easier when weaker
properties are postulated.  The same rationale applies to the
composition operation.  Thus, the following definition of a spatial
calculus, based on abstract partition schemes, contains minimal
requirements.

\begin{defi}
  \label{def:qualitative_calculus}
  A \emph{qualitative calculus} with binary relations
  is a tuple\linebreak $(\Rel,\Int,\breve{},\diamond)$ with the following properties.
  \begin{Itemize}
    \item
      $\Rel$ is a finite, non-empty set of \emph{base relations}.
      The subsets of $\Rel$ are called \emph{relations}.
      We use $r,s,t$ to denote base relations and $R,S,T$ to denote relations.
    \item
      $\Int = (\Univ, \varphi)$ is an \emph{interpretation} with
      a non-empty universe $\Univ$ and
      a map $\varphi : \Rel \to 2^{\Univ \times \Univ}$
      with $(\Univ, \{\varphi(r) \mid r \in \Rel\})$ being a weak partition scheme.
      The map $\varphi$ is extended to arbitrary relations by setting
      $\varphi(R) = \bigcup_{r\in R}\varphi(r)$ for every $R \subseteq \Rel$.
    \item
      The \emph{converse operation} $\breve{~}$ is a map
      $\breve{~} : \Rel \to 2^\Rel$ that satisfies
      \begin{equation}
        \varphi(r\breve{~}) \supseteq \varphi(r)\breve{~}
        \label{eq:abstract_converse}
      \end{equation}
      for every $r \in \Rel$.
      The operation $\breve{~}$ is extended to arbitrary relations by setting
      $R\breve{~} = \bigcup_{r\in R} r\breve{~}$ for every $R \subseteq \Rel$.
    \item
      The \emph{composition operation} $\diamond$ is a map
      $\diamond : \Rel \times \Rel \to 2^\Rel$ that satisfies
      \begin{equation}
        \varphi(r \diamond s) \supseteq \varphi(r) \circ \varphi(s)
        \label{eq:abstract_composition}
      \end{equation}
      for all $r,s \in \Rel$.
      The operation $\diamond$ is extended to arbitrary relations by setting
      $R \diamond S = \bigcup_{r\in R} \bigcup_{s\in S} r \diamond s$ for every $R,S \subseteq \Rel$.
  \end{Itemize}  
\end{defi}

We call Properties \eqref{eq:abstract_converse} and \eqref{eq:abstract_composition}
\emph{abstract converse} and \emph{abstract composition}, following Ligozat's naming \cite{Lig05}.
Our notion of a qualitative calculus makes weaker requirements on the converse
operation than Ligozat and Renz's notions of a weak representation \cite{Lig05,LigozatR04}.
We have already discussed a rationale behind choosing these ``weaker than weak'' variants
and will name another one in Section \ref{sec:reasoning}.
On the other hand, our notion makes stronger requirements on the converse
than Nebel and Scivos's notion of a constraint algebra \cite{DBLP:journals/ki/NebelS02}.
In Section \ref{sec:reasoning}, we will discuss a reason for adopting the minimal requirements we pose to abstract converse and composition.
The following definition gives the stronger variants of converse and composition existing in the literature.

\begin{defi}
  \label{def:stronger_versions_of_comp+conv}
  Let $C = (\Rel,\Int,\breve{},\diamond)$ be a qualitative calculus.
  \par\smallskip\noindent
      $C$ has
      \emph{weak converse} if, for all $r \in \Rel$:
      \begin{equation}
        r\breve{~} = \bigcap\{S \subseteq \Rel \mid \varphi(S) \supseteq \varphi(r)\breve{~}\}
        \label{eq:weak_converse}
      \end{equation}
      $C$ has
      \emph{strong converse} if, for all $r \in \Rel$:
      \begin{equation}
        \varphi(r\breve{~}) = \varphi(r)\breve{~}
        \label{eq:strong_converse}
      \end{equation}
      $C$ has
      \emph{weak composition} if, for all $r,s \in \Rel$:
      \begin{equation}
        r\diamond s = \bigcap\{T \subseteq \Rel \mid \varphi(T) \supseteq \varphi(r)\circ\varphi(s)\}
        \label{eq:weak_composition}
      \end{equation}
      $C$ has
      \emph{strong composition} if, for all $r,s \in \Rel$:
      \begin{equation}
        \varphi(r \diamond s) = \varphi(r) \circ \varphi(s)
        \label{eq:strong_composition}
      \end{equation}
\end{defi}

The following fact captures that Properties \eqref{eq:abstract_converse}--\eqref{eq:strong_composition}
immediately carry over to arbitrary relations;
the straightforward proof is given in
the appendix \ref{app:proofs4Req}.
It has consequences for efficient spatio-temporal reasoning,
which are explained in Section \ref{sec:reasoning}.

\begin{fact}
  \label{fact:weak+strong_conv+comp_general}
  Given a qualitative calculus $(\Rel,\Int,\breve{},\diamond)$
  and relations $R,S \subseteq \Rel$,
  the following hold:
  \begin{align}
    \varphi(R\breve{~}) & \supseteq \varphi(R)\breve{~}         \label{eq:abstract_converse_general}   \\
    \varphi(R \diamond S)  & \supseteq \varphi(R) \diamond \varphi(S) \label{eq:abstract_composition_general}
  \end{align}
      If $C$ has
      weak converse, then, for all $R \subseteq \Rel$:
      \begin{equation}
        R\breve{~} = \bigcap\{S \subseteq \Rel \mid \varphi(S) \supseteq \varphi(R)\breve{~}\} 
        \label{eq:weak_converse_general}
      \end{equation}
      If $C$ has
      strong converse, then, for all $R \subseteq \Rel$:
      \begin{equation}
        \varphi(R\breve{~}) = \varphi(R)\breve{~} 
        \label{eq:strong_converse_general}
      \end{equation}
      If $C$ has
      weak composition, then, for all $R, S \subseteq \Rel$:
      \begin{equation}
        R\diamond S = \bigcap\{T \subseteq \Rel \mid \varphi(T) \supseteq \varphi(R)\circ\varphi(S)\}
        \label{eq:weak_composition_general}
      \end{equation}
      If $C$ has
      strong composition, then, for all $R, S \subseteq \Rel$:
      \begin{equation}
        \varphi(R \diamond S) = \varphi(R) \circ \varphi(S)
        \label{eq:strong_composition_general}
      \end{equation}
\end{fact}

Since base relations are non-empty and JEPD, we have
\begin{fact}\label{fact:phi-injective}
For any qualitative calculus, $\varphi$ is injective.
\end{fact}

Comparing Definitions
\ref{def:JEPD+partition_schemes}--\ref{def:stronger_versions_of_comp+conv}
with the basic notions of a qualitative calculus in \cite{LigozatR04},
a \emph{weak representation} is a calculus with identity relation,
strong converse and abstract composition.  Our basic notion of a
qualitative calculus is more general: it does not require an identity
relation, and it only requires abstract converse and composition.
Conversely, \cite{LigozatR04} are slightly more general than we are,
because the map $\varphi$ need not be injective. However, this extra
generality is not very meaningful: if base relations are JEPD,
$\varphi$ could only be non-injective in giving multiple names to the
empty relation.  Furthermore, in \cite{LigozatR04}, a
\emph{representation} is a weak representation with strong composition
and an injective map $\varphi$.

\subsection{Spatio-Temporal Reasoning}
\label{sec:reasoning}

The most important flavor of spatio-temporal reasoning is constraint-based reasoning.
Like with a classical constraint satisfaction problem (CSP), we are given a set of variables and constraints.
The task of constraint satisfaction is to decide whether there exists a valuation of all variables that satisfies the constraints. 
In calculi for spatio-temporal reasoning, variables range over the specific spatial or temporal domain of a qualitative representation.
The relations defined by the calculus serve as constraint relations. 
More formally, we have:

\begin{defi}[$\text{QCSP}$]
    Let $(\Rel,\Int,\breve{},\diamond)$ be a binary qualitative calculus with $\Int = (\Univ, \varphi)$ and let $X$ be a set of variables ranging over $\Univ$.
     A \emph{qualitative constraint} is a formula 
    $x_i \, R_j\, x_k$ with variables $x_i, x_k\in X$ and relation $R_j \in \Rel$.
    We say that a valuation $\psi : X \to \Univ$ \emph{satisfies} $x_i \, R_j\, x_k$ if $\left(\psi(x_1), \psi(x_k)\right) \in \phi(R_j)$ holds.
    
    A \emph{qualitative constraint satisfaction problem} (QCSP) is the task to decide whether there is a valuation $\psi$ for a set of variables satisfying a set of constraints.
\end{defi}

In the following we use $X$ to refer to the set of variables and $r_{x,y}$ stands for the constraint relation between variables $x$ and $y$. 
For simplicity and wlog.~it is assumed that for every pair of variables exactly one constraint relation is given.

Several techniques originally developed for finite domain CSP can be adapted to spatio-temporal QCSPs.
Since deciding CSP instances is already NP-complete for search problems with finite domains, heuristics are important.
One particular valuable technique is constraint propagation which aims at making implicit constraints explicit in order to identify variable assignments that would violate some constraint. 
By pruning away these variable assignments, a consistent valuation can be searched more efficiently.
A common approach is to enforce $k$-consistency. 

\begin{defi}
A CSP is $k$-consistent if for all subsets of variables $X'\subset X$ with $|X'|=k-1$ we can extend any valuation of $X'$ that satisfies the constraints to a valuation of $X' \cup \{x\}$  also satisfying the constraints, where $x \in X\setminus X'$ is any additional variable. \label{def:k-consistency}
\end{defi}

QCSPs are naturally 1-consistent as the domains are infinite and there are no unary constraints.
A QCSP is 2-consistent if $r_{x,y} = r_{y,x}\breve{~}$ and $r_{x,y} \neq \emptyset$ as relations are typically serial.
A 3-consistent QCSP is also called {\em path-consistent} and Definition \ref{def:k-consistency} can be rewritten using compositions as

\begin{equation}
 r_{x,y} \subseteq \bigcap_{z\in X} r_{x,z} \circ r_{z,y}
\end{equation}
 
and we can enforce the 3-consistency by iterating the refinement operation 

\begin{equation}
 r_{x,y} \gets r_{x,y} \cap r_{x,z} \circ r_{z,y} \label{eq:rel-refinement}
\end{equation}
for all variables $x,y,z \in X$ until a fix point is reached.
This procedure is known as the path-consistency algorithm \cite{dechter}.
For finite constraint networks the algorithm always terminates since the refinement operation is monotone and there are only finitely many relations.


If a qualitative calculus does not provide strong composition, iterating Equation (\ref{eq:rel-refinement}) is not possible as it would lead to relations not contained in \Rel. 
It is however straightforward to weaken Equation (\ref{eq:rel-refinement}) using weak composition.
\begin{equation}
 r_{x,y} \gets r_{x,y} \cap r_{x,z} \diamond r_{z,y} \label{eq:rel-refinement2}
\end{equation}
This procedure is called enforcing {\em algebraic closure} or {\em a-closure} for short.
The reason why, in Definition \ref{def:qualitative_calculus}, we require composition to be at least abstract
is that the underlying inequality guarantees that reasoning via a-closure is sound.

Enforcing $k$-consistency or algebraic closure does not change the solutions of a CSP, as only impossible valuations are removed.
If during application of Equation (\ref{eq:rel-refinement2}) an empty relation occurs, the QCSP is thus known to be inconsistent.
By contrast, an algebraically closed QCSP may not be consistent though. 
However, for several qualitative calculi (or at least sub-algebras thereof) algebraic closure and consistency coincide.

Though we speak about composition in the following two paragraphs, the same statements hold for converse.

Fact \ref{fact:weak+strong_conv+comp_general} has the consequence that the composition operation of a calculus is uniquely determined
if the composition of each pair of base relations is given.
This information is usually stored in a table, the \emph{composition table}.
Then, computing the composition of two arbitrary relations is just a matter of table look-ups which allows algebraic closure to be enforced efficiently.
Speaking in terms of composition tables,
abstract composition implies that each cell corresponding to $r \diamond s$ contains \emph{at least}
those base relations $t$ whose interpretation intersects with $\varphi(r) \circ \varphi(s)$.
In addition, weak composition implies that each cell contains \emph{exactly} those $t$.
If composition is strong, then \Rel and $\varphi$ even have to ensure that
whenever $\varphi(t)$ intersects with $\varphi(r) \circ \varphi(s)$, it is contained in $\varphi(r) \circ \varphi(s)$ --
i.e., the composition of the interpretation of any two base relations has to be the union of interpretations of certain base relations.

\subsection{Existing Qualitative Spatio-Temporal Representations}
This paper is concerned with properties of binary spatio-temporal calculi
that are described in the literature and implemented in the spatial representation and reasoning tool \SparQ
\cite{sparqmanual,WFW+06}.
Table \ref{tab:calculi} lists these calculi.

\begin{table}[t]
  \rowcolors*{1}{lightblue}{}%
  {\centering
    \begin{tabular}{lllrcl}
      \hline\rowcolor{medblue}
      Name                                    & Ref.                                              & Domain                               & \#BR   & \multicolumn{2}{>{\columncolor{medblue}}l}{~RM} \\
      \hline
      9-Intersection                          & \cite{DBLP:conf/ssd/Egenhofer91}                  & simple 2D regions                    &   8    & ~I & \cite{GPP95,KPWZ10}                        \\
      Allen's interval relations              & \cite{allen:83}                                   & intervals (order)                    &  13    & ~A & \cite{Vilain:1989qv}                       \\
      Block Algebra                           & \cite{DBLP:journals/logcom/BalbianiCC02}          & $n$-dimensional blocks               & $13^n$ & ~A & \cite{DBLP:journals/logcom/BalbianiCC02}   \\
      Cardinal Dir.\ Calculus  CDC            & \cite{DBLP:conf/ogai/Frank91,ligozat-JVLC:98}     & directions (point abstr.)            &   9    & ~A & \cite{ligozat-JVLC:98}                     \\
      Cardinal Dir.\ Relations CDR            & \cite{DBLP:journals/ai/SkiadopoulosK04}           & regions                              & 218    & ~P &                                            \\
      CycOrd, binary CYC${}_\text{b}$         & \cite{DBLP:journals/ai/IsliC00}                   & oriented lines                       &   4    & ~U &                                            \\
      Dependency Calculus                     & \cite{ragni-scivos-KI:05}                         & points (partial order)               &   5    & ~A & \cite{ragni-scivos-KI:05}                  \\
      Dipole Calculus\myfnm{a} DRA$_\text{f}$ & \cite{moratz-renz-wolter-ECAI:00,MoratzEtAl2011}  & directions from line segm.\          &  72    & ~I & \cite{WL10}                                \\
      \quad DRA$_{\text{fp}}$                 & \cite{MoratzEtAl2011}                             & directions from line segm.\          &  80    & ~I &                                            \\
      \quad DRA-connectivity                  & \cite{wallgruen-etal-ACMGIS:10}                   & connectivity of line segm.\          &   7    & ~U &                                            \\
      Geometric Orientation                   & \cite{DyL10}                                      & relative orientation                 &   4    & ~U &                                            \\
      INDU                                    & \cite{pujari-sattar-IJCAI:99}                     & intervals (order, rel.\ dur.n)       &  25    & ~P &                                            \\
      \tworows
      OPRA$_m$, $m = 1,\dots,8$               & \cite{Moratz06_ECAI,MossakowskiMoratz2011}        & \multicolumn{2}{>{\columncolor{white}}l}{oriented points \hspace*{3.5mm} $4m\cdot(4m+1)$} &     \\
      (Oriented Point Rel.\ Algebra)          &                                                   &                                      &        & ~I & \cite{WL10}                                \\
      \showrowcolors
      Point Calculus                          & \cite{Vilain:1989qv}                              & points (total order)                 &   3    & ~A & \cite{Vilain:1989qv}                       \\
      Qualitat.\ Traject.\ Calc.\ QTC$_{\text{B11}}$ & \cite{Weghe04,DBLP:conf/geos/WegheKBM05}   & moving point obj.s in 1D             &   9    & ~U &                                            \\
      \quad QTC$_{\text{B12}}$                & ''                                                & ''                                   &  17    & ~U &                                            \\
      \quad QTC$_{\text{B21}}$                & ''                                                & moving point obj.s in 2D             &   9    & ~U &                                            \\
      \quad QTC$_{\text{B22}}$                & ''                                                & ''                                   &  27    & ~U &                                            \\
      \quad QTC$_{\text{C12}}$                & ''                                                & ''                                   &  81    & ~U &                                            \\
      \quad QTC$_{\text{C22}}$                & ''                                                & ''                                   & 305    & ~U &                                            \\
      Region Connection Calc.\ RCC-5          & \cite{randell-cui-cohn-KR:92}                     & regions                              &   5    & ~A & \cite{JD97}                                \\
      \quad RCC-8                             & \cite{randell-cui-cohn-KR:92}                     & regions                              &   8    & ~A & \cite{renz:02}                             \\
      Rectangular Cardinal Rel.s RDR          & \cite{NavarreteEtAl13}                            & regions                              &  36    & ~A & \cite{NavarreteEtAl13}                     \\
      Star Algebra STAR${}_4$                 & \cite{renz-mitra-PRICAI:04}                       & directions from a point              &   9    & ~P &                                            \\
      \hline
    \end{tabular}%
  }
  \myfn{a}{Variant DRA$_\text{c}$ is not based on a weak partition scheme -- JEPD is violated \cite{MoratzEtAl2011}.}
  \par\smallskip
  \rowcolors*{1}{}{}%
  \begin{tabular}{l@{~~}l}
    \#BR: & number of base relations \\
    RM:   & reasoning method used to decide consistency of CSPs with base relns only: \\
          & \textbf{A}-closure;\, \textbf{P}olynomial: reducible to linear programming;\\
          & \textbf{I}ntractable (assuming P $\neq$ NP);\, \textbf{U}nknown
  \end{tabular}

  \par\smallskip
  \addtolength{\belowcaptionskip}{-10pt} 
  \caption{Overview of the binary calculi tested.}
  \label{tab:calculi}
\end{table}

\section{Relation Algebras}
\label{sec:relation_algebras}

\subsection{Definition}

If we focus our attention on spatio-temporal calculi with binary relations,
it is reasonable to ask whether they are relation algebras (RAs).
If a calculus is a RA, it is guaranteed to have properties
that allow several optimizations in constraint reasoners.
For example, associativity of the composition operation $\diamond$ ensures that,
if the reasoner encounters a path $ArBsCtD$ of length 3,
then the relation between $A$ and $D$ can be computed ``from left to right''.
Without associativity, $(r \diamond s) \diamond t$ as well as $r \diamond (s \diamond t)$
would have to be computed.
RAs have been considered in the literature for spatio-temporal calculi \cite{LigozatR04,Due05,Mos07}. 

An (abstract) RA is defined in \cite{Mad06}; here we use the symbols $\cup$, $\diamond$, and $\id$ instead of $+$, $;$, and $1'$.
Let $A$ be a set containing $\id$ and $1$, and let $\cup$, $\diamond$ be binary and $\bar{~}$, $\breve{~}$ unary operations on $A$.
The relevant axioms (\RA{1}--\RA{10}, \WA, \SA, and \PL) are given in Table \ref{tab:relation_algebra_axioms}.
All axioms except \PL can be weakened to only one of two inclusions,
which we denote by a superscript ${}^\supseteq$ or ${}^\subseteq$.
For example, $\RA[sup]{7}$ denotes $(r\breve{~})\breve{~} \supseteq r$.
Likewise, we use \PL[right] and \PL[left].
Then,
$\mathfrak{A} = (A,\cup,\bar{~},\diamond,\breve{~},\id)$ is a
\begin{Itemize}
  \item
    \emph{non-associative relation algebra (NA)} if it satisfies Axioms \RA1--\RA3, \RA5--\RA{10};
  \item
    \emph{semi-associative relation algebra (SA)} if it is an NA and satisfies
    Axiom \axiom{SA},
  \item
    \emph{weakly associative relation algebra (WA)} if it is an NA and satisfies
    \axiom{WA},
  \item 
    \emph{relation algebra (RA)} if it satisfies \RA1--\RA{10},
\end{Itemize}
for all $r,s,t \in A$.
Every RA is a WA; every WA is an SA; every SA is an NA. 
\begin{table}[t]
  \addtolength{\belowcaptionskip}{-16pt}
  \rowcolors{1}{lightblue}{}%
  \centering
  \begin{tabular}{lrcll}
    \hline
    \RA1       & $r \cup s$                                                    & $=$ & $s \cup r$                           & $\cup$-commutativity                  \\
    \RA2       & $r \cup (s \cup t)$                                           & $=$ & $(r \cup s) \cup t$                  & $\cup$-associativity                  \\
    \RA3       & $\overline{\bar r \cup \bar s} \cup \overline{\bar r \cup s}$ & $=$ & $r$                                  & Huntington's axiom                    \\
    \RA4       & $r \diamond (s \diamond t)$                                   & $=$ & $(r \diamond s) \diamond t$          & $\diamond$-associativity                 \\
    \RA5       & $(r \cup s) \diamond t$                                       & $=$ & $(r \diamond t) \cup (s \diamond t)$ & $\diamond$-distributivity                \\
    \RA6       & $r \diamond \id$                                              & $=$ & $r$                                  & identity law                          \\
    \RA7       & $(r\breve{~})\breve{~}$                                       & $=$ & $r$                                  & $\breve{~}$-involution                \\
    \RA8       & $(r \cup s)\breve{~}$                                         & $=$ & $r\breve{~} \cup s\breve{~}$         & $\breve{~}$-distributivity            \\
    \RA9       & $(r \diamond s)\breve{~}$                                     & $=$ & $s\breve{~} \diamond r\breve{~}$     & $\breve{~}$-involutive distributivity \\
    \RA{10}    & $r\breve{~} \diamond \overline{r \diamond s} \cup \bar s$     & $=$ & $\bar s$                             & Tarski/de Morgan axiom                \\
    \hline
    \WA        & $((r \cap \id) \diamond 1) \diamond 1$                        & $=$ & $(r \cap \id) \diamond 1$            & weak $\diamond$-associativity            \\
    \SA        & $(r \diamond 1) \diamond 1$                                   & $=$ & $r \diamond 1$                       & $\diamond$ semi-associativity            \\
    \RA{6l}    & $\id \diamond r$                                              & $=$ & $r$                                  & left-identity law                     \\
    \PL        & $(r \diamond s) \cap t\breve{~} = \emptyset$                  & $\Leftrightarrow$ & $(s \diamond t) \cap r\breve{~} = \emptyset$ 
                                                                                                                            & Peircean law \\
    \hline
  \end{tabular}
  \par\smallskip
  \caption{%
    Axioms for relation algebras and weaker variants \cite{Mad06}.\protect\\
  }
  \label{tab:relation_algebra_axioms}
\end{table}

In the literature, a different axiomatization is sometimes used, for example in \cite{LigozatR04}.
The most prominent difference is that \RA{10} is replaced by \PL,
``a more intuitive and useful form, known as the Peircean law or De Morgan's Theorem K'' \cite{HH02}.
It is shown in \cite[Section 3.3.2]{HH02} that, given \RA1--\RA3, \RA5, \RA7--\RA9,
the axioms \RA{10} and \PL are equivalent. The implication $\PL \Rightarrow \RA{10}$ does not need \RA5 and \RA8.

Furthermore, Table \ref{tab:relation_algebra_axioms} contains the redundant axiom \RA{6l}
because it may be satisfied when some of the other axioms are violated. It is straightforward to establish
that \RA6 and \RA{6l} are equivalent given \RA7 and \RA9, see 
the appendix \ref{app:proofs4RA}.

Due to our minimal requirements to a qualitative calculus given in Def.\ \ref{def:qualitative_calculus},
certain axioms are always satisfied; see 
the appendix 
for a proof of the following%
\begin{fact}
  \label{fact:minimal_axiom_set_implied_by_calculus_def}
  Every qualitative calculus satisfies
  \RA1--\RA3, \RA5, \RA[sup]7, \RA8, \WA[sup], \SA[sup]
  for all (base and complex) relations.
  This axiom set is maximal: each of the remaining axioms
  in Table \ref{tab:relation_algebra_axioms} is not satisfied by some
  qualitative calculus.
\end{fact}

\subsection{Discussion of the Axioms}
\label{sec:axiom_discussion}
We will now discuss the relevance of the above axioms for spatio-temporal representation and reasoning.
Due to Fact \ref{fact:minimal_axiom_set_implied_by_calculus_def}, we only need to consider
axioms \RA4, \RA6, \RA7, \RA9, \RA{10} (or \PL) and 
their weakenings \RA{6l}, \SA, \WA.

\myparagraph{\RA4 (and \SA, \WA).}
Axiom \RA4 is helpful for modeling.
It allows for writing chains of compositions without parentheses,
which have an unambiguous meaning.
For example, consider the following statement in natural language about
the relative length and location of two intervals $A$ and $D$.
  \emph{
  Interval $A$ is before some equally long interval that is contained in some longer interval
  that meets the shorter $D$.%
  }
This statement is just a conjunction of relations between $A$, the unnamed intermediary intervals $B,C$,
and $D$. When we evaluate it,
it intuitively does not matter whether we give priority to the composition of the relations between $A,B$ and $B,C$
or to the composition of the relations between $B,C$ and $C,D$.

However, INDU does not satisfy Axiom \RA4 and, therefore, here the two ways of parenthesizing the above statement
lead to different relations between $A$ and $D$. This behavior is sometimes attributed to the absence of strong composition,
which we will refute in Section \ref{sec:properties}.
Conversely, strong composition implies \RA4 since composition of binary relations over \Univ is associative:
\begin{fact}
  Let $C = (\Rel,\Int,\breve{},\diamond)$ be a qualitative calculus with strong composition.
  Then $C$ satisfies \RA4.
\end{fact}
Note that INDU still satisfies the weakenings \SA and \WA of \RA4,
and we already know from Fact \ref{fact:minimal_axiom_set_implied_by_calculus_def}
that the inequalities \SA[sup] and \WA[sup] are always satisfied. 

\smallskip
Furthermore, Axiom \RA4 is useful for optimizing reasoning algorithms:
suppose a scenario that contains the constraints $\{WrX, XsY, YtZ, Wr'Z\}$ with variables $W,X,Y,Z$ needs to be checked for consistency.
If \emph{one} of the inclusions \RA[sup]4 and \RA[sub]4 is satisfied
-- say, $r \diamond (s \diamond t) \subseteq (r \diamond s) \diamond t$ --
then it suffices to compute the ``finer'' composition result $r \diamond (s \diamond t)$
and check whether it contains $r'$.
Otherwise, both results have to be computed and checked for containment of $r'$.

In case neither \RA[sup]4 not \RA[sub]4 is satisfied,
it is possible to establish associativity,
at the cost of coarsening the composition operation and thereby reducing the ``information content'' inherent in the composition table,
but with the benefit of providing a sound approximation with the advantages discussed so far.
This can be done as follows.
If there is some triple $(r,s,t)$ of base relations that violates associativity, for example, $r \diamond (s \diamond t) \subsetneq (r \diamond s) \diamond t$,
then enrich the table entries that are used to compute $r \diamond (s \diamond t)$ with the base relations that are missing to reach $r \diamond (s \diamond t) = (r \diamond s) \diamond t$.
Repeat this step until no more violations are found. This procedure terminates -- in the worst case, it runs until all entries contain the universal relation.
The resulting table represents an abstract composition operation
that satisfies associativity, but its ``information content'' is lower than that of the original table because, by adding relations to certain entries,
the disjunctions represented by these entries have been enlarged.

\myparagraph{\RA6 and \RA{6l}.}
Axioms \RA6 and \RA{6l} do not seem to play a significant role in (optimizing) satisfiability checking,
but the presence of an \id relation is needed for the standard reduction from the correspondence problem to satisfiability:
to test whether a constraint system admits the equality of two variables $x,y$,
one can add an \id-constraint between $x,y$ and test the extended system for satisfiability.

Furthermore, the absence of an \id relation may lead to an earlier loss of precision.
For example, assume two variants of the 1D Point Calculus \cite{Vilain:1989qv}:
PC$_{=}$ with the relations \emph{less than} ($<$), \emph{equal} ($=$), and \emph{greater than} ($>$),
interpreted as the natural relations $<,=,>$ over the domain of the reals,
and its approximation PC$_{\approx}$ with the relations \emph{less than} ($<$),
\emph{approximately equal} ($\approx$), and \emph{greater than} ($>$),
where $\approx$ is interpreted as the set of pairs of points whose distance is below a certain threshold.
Then, $=$ is the \id-relation of PC$_{=}$ and $=\diamond=$ results in $\{=\}$,
whereas PC$_{\approx}$ has no \id-relation and $\approx\diamond\approx$ results in the universal relation.

\myparagraph{\RA7 and \RA9.}
These axioms allow for certain optimizations in decision procedures for satisfiability based
on algebraic operations like algebraic closure.
If \RA7 holds, the reasoning system does not need to store both constraints $A\,r\,B$ and $B\,r'\,A$,
since $r'$ can be reconstructed as $r\breve{~}$ if needed. 
Similarly, \RA9 grants that, when enforcing algebraic closure by using Equation (\ref{eq:rel-refinement2})
to refine constraints between variable $A$ and $B$, it is sufficient to compute composition
once and, after applying converse, reuse it to refine the constraint between $B$ and $A$ too. 


Current reasoning algorithms and their implementations use the described optimizations;
they produce incorrect results for calculi violating \RA7 or \RA9.

\myparagraph{\RA{10} and \PL.}
These axioms reflect that the relation symbols of a calculus
indeed represent binary relations, i.e., pairs of elements of a universe.
This can be explained from two different points of view.
\begin{enumerate}
  \item
    If binary relations are considered as sets, \RA{10} is equivalent to 
    $
      r\breve{~} \diamond \overline{r \diamond s} \subseteq \bar s.
    $
    If we further assume the usual set-theoretic interpretation of the composition of two relations,
    the above inclusion reads as:
    \emph{%
      For any $X,Y$,
      if $Z\,r\,X$ for some $Z$ and, $Z\,r\,U$ implies not $U\,s\,Y$ for any $U$,
      then not $X\,s\,Y$.}
    This is certainly true because $X$ is one such $U$.
  \item
    Under the same assumptions,
    each side of \PL says (in a different order) that there can be no triangle $X\,r\,Y, Y\,s\,Z, Z\,t\,X$.
    The equality then means that the ``reading direction'' does not matter, see also \cite{Due05}.
    This allows for reducing nondeterminism in the a-closure procedure,
    as well as for efficient refinement and enumeration of consistent scenarios.
\end{enumerate}

\subsection{Prerequisites for Being a Relation Algebra}
The following correspondence between properties of a calculus
and notions of a relation algebra is due to Ligozat and Renz \cite{LigozatR04}.
\begin{proposition}
  Every calculus $C$ based on a partition scheme is an NA.
  If, in addition, the interpretations of the base relations are serial,
  then $C$ is an SA.
\end{proposition}

Furthermore, \RA7 is equivalent to the requirement that a calculus has strong converse.
This is captured by the following lemma.

\begin{lemma}
  \label{lem:R7_and_strong_converse}
  Let $C = (\Rel,\Int,\breve{},\diamond)$ be a qualitative calculus.
  Then the following properties are equivalent.
      \begin{Enumerate}
        \item
          $C$ has strong converse.
        \item
          Axiom \RA{7} is satisfied for all base relations $r \in \Rel$.
        \item
          Axiom \RA{7} is satisfied for all relations $R \subseteq \Rel$.
      \end{Enumerate}
\end{lemma}

\begin{proof}
      Items (2) and (3) are equivalent due to distributivity of $\breve{~}$ over $\cup$,
      which is introduced with the cases for non-base relations in Definition \ref{def:qualitative_calculus}.
      
      For ``(1) $\Rightarrow$ (2)'',
      the following chain of equalities, for any $r \in \Rel$, is due to $C$ having strong converse:
      $
          \varphi(r\breve{~}\breve{~}) = \varphi(r\breve{~})\breve{~} = \varphi(r)\breve{~}\breve{~} = \varphi(r).
      $
      Since \Rel is based on JEPD relations and $\varphi$ is injective, this implies that $r\breve{~}\breve{~} = r$.
      
      For ``(2) $\Rightarrow$ (1)'',
      we show the contrapositive.
      Assume that $C$ does not have strong converse.
      Then $\varphi(r\breve{~}) \supsetneq \varphi(r)\breve{~}$, for some $r \in \Rel$;
      hence $\varphi(r\breve{~})\breve{~} \supsetneq \varphi(r)\breve{~}\breve{~}$.
      We can now modify the above chain of equalities replacing the first two equalities
      with inequalities, the first of which is
      due to Requirement \eqref{eq:abstract_converse} in the definition of the converse (Def.\ \ref{def:qualitative_calculus}):
      $
          \varphi(r\breve{~}\breve{~}) \supseteq \varphi(r\breve{~})\breve{~} \supsetneq \varphi(r)\breve{~}\breve{~} = \varphi(r).
      $
      Since $\varphi(r\breve{~}\breve{~}) \neq \varphi(r)$,
      we have that $r\breve{~}\breve{~} \neq r$.
      \qed
\end{proof}

\section{Algebraic Properties of Existing Calculi}
\label{sec:properties}

In this section, we report on tests for algebraic properties we have performed on spatio-temporal calculi.
We want to answer the following questions.
\emph{(1) Which existing calculi correspond to relation algebras?}
\emph{(2) Which weaker notions of relation algebras correspond to calculi that do not fall under (1)?}

We examined the corpus of the 31 calculi\footnote{For the parametrized calculi DRA, OPRA, QTC, we count every variant separately.}
listed in Table \ref{tab:calculi}. This selection is restricted to calculi with
(a) binary relations -- because the notion of a relation algebra is best understood for binary relations --
and (b) an existing implementation in \SparQ.

To answer Questions (1) and (2), we use the axioms for relation algebras listed in Table \ref{tab:relation_algebra_axioms}
using both the heterogeneous tool set \Hets \cite{MML07} and \SparQ.
Due to Fact \ref{fact:minimal_axiom_set_implied_by_calculus_def},
it suffices to test Axioms \RA4, \RA6, \RA7, \RA9, \RA{10} (or \PL)
and, if necessary, the weakenings \SA, \WA, and \RA{6l}.
The weakenings are relevant 
to capture weaker notions such as semi-associative or weakly associative algebras,
or algebras that violate either \RA6 or some of the axioms that imply the equivalence
of \RA6 and \RA{6l}.
Because all axioms except \RA{10} contain only operations that distribute over the union $\cup$,
it suffices to test them for base relations only.
Therefore, we have written a CASL specification of \RA4, \RA6, \RA7, \RA9, \PL, \SA, \WA, and \RA{6l},
and used a \Hets parser that reads the definitions of the above listed calculi in \SparQ
to test them against our CASL specification.
In addition, we have tested all definitions against \RA4, \RA6, \RA7, \RA9, \PL, and \RA{6l}
using \SparQ's built-in function \texttt{analyze-calculus}.

A part of the calculi have already been tested
by Florian Mossakowski \cite{Mos07}, 
using a different CASL specification based on an equivalent axiomatization from \cite{LigozatR04}.
He comprehensively reports on the outcome of these tests, and on repairs made to the composition table where possible.

The results of our and Mossakowski's tests are summarized in Table \ref{tab:calculi_tests}; details are listed in 
the appendix.
\begin{table}[t]
  \addtolength{\belowcaptionskip}{-16pt}
  \centering
  \begin{small}
    \rowcolors{1}{lightblue}{}%
    \setcounter{myfn}{0}%
    \begin{tabular}{ll*{9}c}
      \hline\rowcolor{medblue}\rule{0pt}{9pt}%
      Calculus                                & Tests\myfnm{fnTests} 
                                                     & ~\RA4~ & ~\SA~ & ~\WA~ & ~\RA6~ & ~\RA{6l}~ & ~\RA7~ & ~\RA9~ & ~\PL~ & ~\RA{10}~ \\
      \hline\rule{0pt}{9pt}%
      Allen                                   & \MHS & \YES   & \YES  & \YES  & \YES   & \YES      & \YES   & \YES   & \YES  & \YES      \\
      Block Algebra                           & \HS  & \YES   & \YES  & \YES  & \YES   & \YES      & \YES   & \YES   & \YES  & \YES      \\
      Cardinal Direction \emph{Calculus}      & \MHS & \YES   & \YES  & \YES  & \YES   & \YES      & \YES   & \YES   & \YES  & \YES      \\
      CYC${}_\text{b}$, Geometric Orientation & \HS  & \YES   & \YES  & \YES  & \YES   & \YES      & \YES   & \YES   & \YES  & \YES      \\
      DRA${}_{\text{fp}}$, DRA-conn.          & \HS  & \YES   & \YES  & \YES  & \YES   & \YES      & \YES   & \YES   & \YES  & \YES      \\
      Point Calculus                          & \HS  & \YES   & \YES  & \YES  & \YES   & \YES      & \YES   & \YES   & \YES  & \YES      \\
      RCC-5, Dependency Calc.                 & \MHS & \YES   & \YES  & \YES  & \YES   & \YES      & \YES   & \YES   & \YES  & \YES      \\
      RCC-8, 9-Intersection                   & \MHS & \YES   & \YES  & \YES  & \YES   & \YES      & \YES   & \YES   & \YES  & \YES      \\
      STAR${}_4$                              & \HS  & \YES   & \YES  & \YES  & \YES   & \YES      & \YES   & \YES   & \YES  & \YES      \\
      \hline\rule{0pt}{9pt}%
      DRA${}_{\text{f}}$                      & \MHS & ~~\,19 & \YES  & \YES  & \YES   & \YES      & \YES   & \YES   & \YES  & \YES      \\
      INDU                                    & \MHS & ~~\,12 & \YES  & \YES  & \YES   & \YES      & \YES   & \YES   & \YES  & \YES      \\
      OPRA${}_n$, $n\leqslant8$               & \MHS & \PC21--\PC91\myfnm{fnOPRA}
                                                              & \YES  & \YES  & \YES   & \YES      & \YES   & \YES   & \YES  & \YES      \\
      \hline\rule{0pt}{9pt}%
      QTC$_{\text{Bxx}}$                      & \MHS & \YES   & \YES  & \YES  & \multicolumn{2}{c}{89--100} 
                                                                                                   & \YES   & \YES   & \YES  & \YES      \\
      \hline\rule{0pt}{9pt}%
      QTC$_{\text{C21}}$                      &  \HS & ~~\,55 & \YES  & \YES  & 99     & 99        & \YES   & ~\,2   & $<$1  & ~\,1      \\
      QTC$_{\text{C22}}$                      &  \HS & ~~\,79 & \YES  & \YES  & 99     & 99        & \YES   & ~\,3   & $<$1  & ~\,1      \\
      \hline\rule{0pt}{9pt}%
      Rectang.\ Direction Relations         & \HS  & \YES   & \YES  & \YES  & 97     & 92        & 89     & 66     & ~\,7  & 52        \\
      \hline\rule{0pt}{9pt}%
      Cardinal Direction \emph{Relations}  & \HS  & ~~\,28 & 17    & \YES  & 99     & 99        & 98     & 12     & $<$1  & 88        \\
      \hline
    \end{tabular}
    \par
    \rowcolors{1}{}{}%
    \begin{tabular}{p{.88\textwidth}}
      \myfn{fnTests}{calculus was tested by: M $=$ \cite{Mos07}, H $=$ \Hets, S $=$ \SparQ} \\
      \myfn{fnOPRA}{21\%, 69\%, 78\%, 83\%, 86\%, 88\%, 90\%, 91\% for OPRA${}_n$, $n = 1,\dots,8$}           \\
    \end{tabular}
    \par
  \end{small}
  \par\smallskip
  \caption{%
    Overview of calculi tested and their properties. The symbol ``\Yes'' means that the axiom is satisfied;
    otherwise the percentage of counterexamples (relations, pairs or triples violating the axiom) is given.%
  }
  \label{tab:calculi_tests}
\end{table}
With the exceptions of QTC, Cardinal Direction Relations (CDR) and Rectangular Direction Relations (RDR),
all tested calculi are at least semi-associative relation algebras;
most of them are even relation algebras.
Hence, these calculi enjoy the advantages for representation and reasoning optimizations discussed in Section \ref{sec:axiom_discussion}.
In particular, current reasoning procedures, which already implement the optimizations described for \RA7 and \RA9,
yield correct results for these calculi, and they could be optimized further by implementing the optimizations described
for \RA4, \RA{10}, and \PL.

The three groups of calculi that are SAs but not RAs are the Dipole Calculus variants DRA${}_{\text{f}}$
(variants DRA${}_{\text{fp}}$ and DRA-connectivity are even RAs!), as well as INDU and OPRA${}_m$ for $m=1,\dots,8$.
These calculi do not even satisfy one of the inclusions $\RA[sup]4$ and $\RA[sub]4$,
which implies that the reasoning optimizations described in Section \ref{sec:axiom_discussion} for Axiom \RA4
cannot be applied, but this is the only disadvantage of these calculi over the others.
Our observations suggest that the meaning of the letter combination ``RA'' in the abbreviations
``DRA'' and ``OPRA'' should stand for ``Reasoning Algebra'', not for ``Relation Algebra''.

In principle, it cannot be completely ruled out that associativity is reported to be violated due to errors in either
the implementation of the respective calculus or the experimental setup. This even applies to non-violations,
although it is much more likely that errors cause sporadic violations than systematic non-violations.
In the case of DRA${}_{\text{f}}$, INDU and OPRA${}_m$, $m=1,\dots,8$,
the relatively high percentage of violations make implementation errors seem unlikely to be the cause.
However, to obtain certainty that these calculi indeed violate \RA4,
one has to find concrete counterexamples and verify them using the original definition of the respective calculus.
For DRA${}_{\text{f}}$ and INDU, this has been done in the literature \cite{MoratzEtAl2011,BCL06}.
Interestingly, the violation of associativity has been attributed to the absence of strong converse and strong composition, respectively.
We remark, however, that the latter cannot be responsible
because, for example, DRA${}_{\text{fp}}$ has an associative, but only weak, composition operation. While DRA${}_{\text{fp}}$ has been proven to be associative
due to strong composition in \cite{MoratzEtAl2011},
for OPRA${}_m$, it can be shown that \emph{none} of the variants for any $m$ are associative (see \cite{MossakowskiMoratzLuecke}).

The $B$-variants of QTC violate only the identity law \RA6 and \RA{6l}.
As observed in \cite{Mos07}, it is possible to equip them with a new \id relation,
modify the interpretation of the other relations such that they become JEPD,
and adapt the converse and composition table accordingly. The thus modified calculi
are then relation algebras.

The $C$-variants of QTC additionally violate \RA4, \RA9, \RA{10}, and \PL.
We call the corresponding notion of algebra semi-associative Boolean algebra with converse-involution.
As a consequence, most of the reasoning optimizations described in Section \ref{sec:axiom_discussion}
cannot be applied to the $C$-variants of QTC; hence, reasoning with these calculi is expected to be less efficient
than with the calculi described so far.
It is possible that the noticeably few violations of \RA9, \RA{10}, and \PL are due to errors in the composition table;
the non-trivial verification is part of future work.

Cardinal Direction Relations and Rectangular Direction Relations are the only calculi with weak converse that we have tested.
The former satisfies only \WA in addition to the axioms that are always satisfied by a Boolean algebra with distributivity.
We call the corresponding notion of algebra weakly associative Boolean algebra.
Hence, this calculus enjoys none of the advantages for representation and reasoning discussed in Section \ref{sec:axiom_discussion}.
Similarly to the $C$-variants of QTC, the relatively small number of violations of \PL may be due to errors in the implementation.
Rectangular Direction Relations additionally satisfies \RA4 and therefore corresponds to what we call an associative Boolean algebra.
Since both calculi satisfy neither \RA7 nor \RA9, current reasoning algorithms and their implementations
yield incorrect results for them, as seen in Section \ref{sec:axiom_discussion}.

An overview of the algebra notions identified is given in Figure \ref{fig:algebra_notions}.
\begin{figure}[t]
\centering
\addtolength{\abovecaptionskip}{-10pt}
  \includegraphics[width=0.9\textwidth,trim=37 264 37 0,clip]{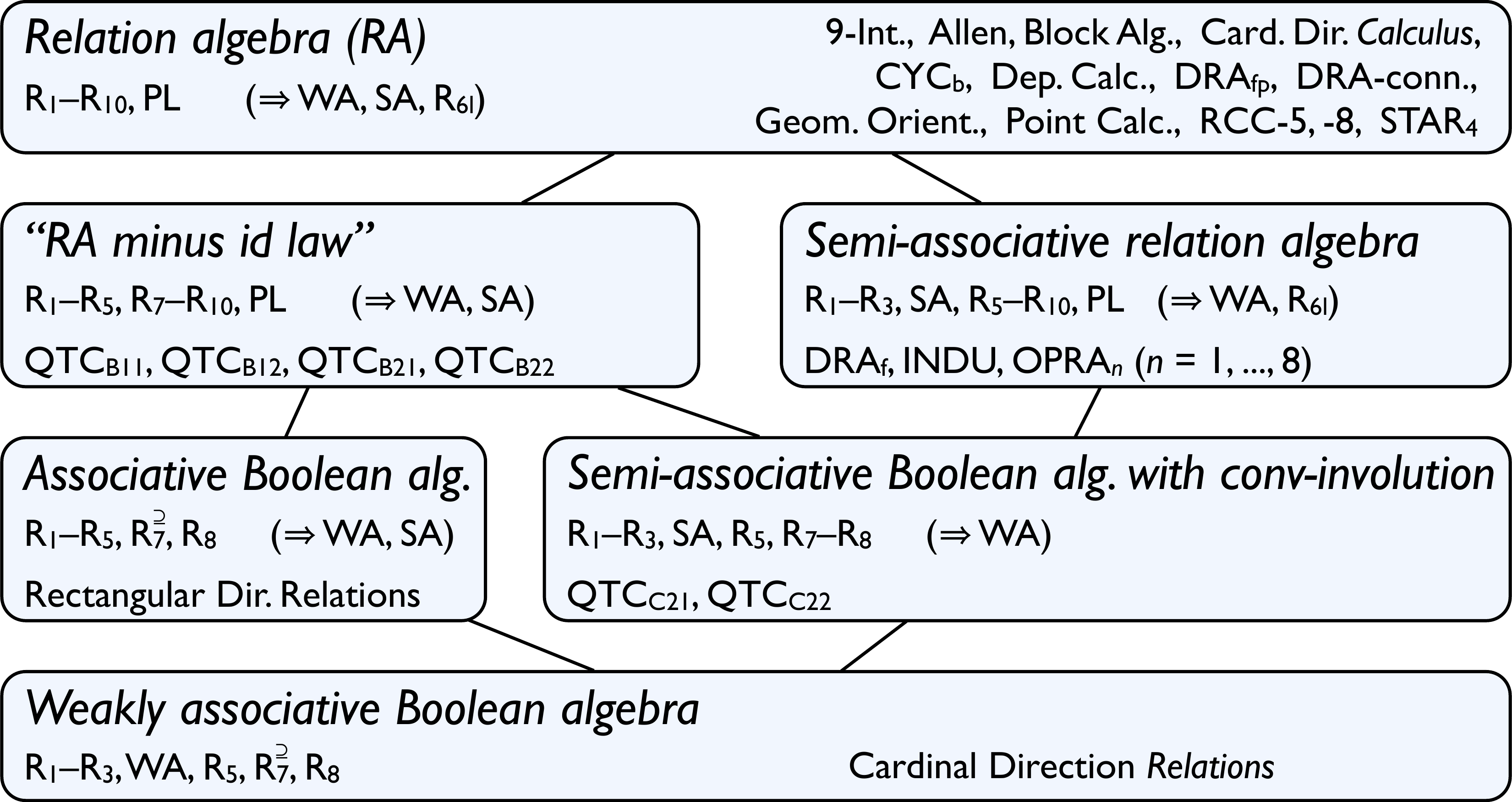}
  \caption{Overview of algebra notions and calculi tested}
  \label{fig:algebra_notions}
\end{figure}

\begin{algorithm}
\caption{Universal algebraic closure algorithm \label{alg:aclosure}}
\begin{algorithmic}[1]
\Function{lookup}{$(C,i,j,s)$}
	\If{$s \vee (i<j)$}
		\State{
		\Return $C_{i,j}$}
	\Else
		\State{\Return $C_{j,i}\breve{~}$}
	\EndIf
\EndFunction 
\vspace*{1ex}
\Function{revise}{$(C,i,j,k,s)$}
		\State{$u \gets \mathbf{false}$} \Comment{update flag to signal whether relation was updated}
		\State{$r \gets C_{i,j} \cap \mathsc{lookup}(C,i,k,s) \diamond \mathsc{lookup}(C,k,j,s)$}
		\If{$s\ \vee$ \RA{9} does not hold}\label{alg:ra9start}
		\State{$r' \gets \mathsc{lookup}(C,j,i,s) \cap (\mathsc{lookup}(C,j,k,s) \diamond \mathsc{lookup}(C,k,i,s))$}
		\State{$r \gets r \cap r'\breve{~}$}
		\State{$r' \gets r' \cap r\breve{~}$}
		\If{$r'\neq C_{j,i}$}
			\State{\textbf{assert} $r'\neq\emptyset$}\Comment{stop if inconsistency is detected}
			\State{$u \gets \mathbf{true}$}
			\State{$C_{j,i} \gets r'$}
		\EndIf \label{alg:ra9end}
		\EndIf
		\If{$r\neq C_{i,j}$}
			\State{\textbf{assert} $r\neq\emptyset$}\Comment{stop if inconsistency is detected}
			\State{$u \gets \mathbf{true}$}
			\State{$C_{i,j} \gets r$}
		\EndIf	\State{\Return $(C,u)$}
\EndFunction
\vspace*{1ex}
\Function{a-closure}{$(n,\{x_1\, r_1\, y_1, \ldots , x_m\, r_m\, y_m\})$}
\If{\RA{7} does not hold}
	\State{$s \gets \mathbf{true}$}  \Comment{without \RA{7} we must store converse relations}
	\State{$C_{i,j} \gets \Univ, i=1,\ldots,n, j=1,\ldots,n$}
\Else
	\State{$s \gets \mathbf{false}$} \Comment{for small calculi/CSPs storing converses may be more efficient}
	\State{$C_{i,j} \gets \Univ, i=1,\ldots,n, j=i+1,\ldots,n$}\Comment{use triangular matrix storage}
\EndIf
	\State{$C_{i,i} \gets \id, i=1,\ldots,n$}
\For{$i=1,\ldots, m$} 
	\State{$x\gets x_i,\; r\gets r_i,\; y\gets y_i$} \Comment{process constraint $x_i\,r_i\,y_i$}
	\If {$\neg s \wedge (x > y)$}
		\State{$(x,y) \gets (y,x),\, r \gets r\breve{~}$} \Comment{only write into upper half of matrix}
	\EndIf
	\State{$C_{x,y} \gets C_{x,y} \cap r$}
			\State{\textbf{assert} $(x=y) \to (\id \in C_{x,y})$}
\EndFor
%
\State{$\mathit{update} \gets \mathbf{true}$}
%
\While{$\mathit{update}$} 
	\State{$\mathit{update} \gets \mathbf{false}$}
	\For{$i=1,\ldots ,n, j=i\!+\!1,\ldots, n, k=1,\ldots n, k\neq i, k\neq j$}
		\State{$(u,C) \gets \mathsc{revise}(C,i,j,k,s)$}
		\State{$\mathit{update} \gets \mathit{update}\vee u$}
	\EndFor
\EndWhile
\State{\Return{$C$}}\Comment{fix point reached}
\EndFunction
%
\end{algorithmic}
\end{algorithm}


When making using of algebraic closure as inference mechanism it is essential to acknowledge that
some axiom violations require special procedures in order to compute algebraic closure. 
Our analysis reveals that there indeed exist calculi that do not meet axioms that have been taken for granted. For example, the current version of GQR can fail to compute algebraic closure correctly for calculi that violate \RA{9}. 
In Algorithm \ref{alg:aclosure} we present a universal algorithm to compute algebraic closure.
For clarity and brevity of the presentation we stick to the well-known but simple control structure of PC-1. 
A real implementation would use an advanced control structure to avoid unnecessary invocations of the \textsc{revise} function, i.e., to use at least PC-2 \cite{mackworth-AI:77}.
Conformance with \RA{7} allows CSP storage to be restricted (flag $s$ in the algorithm), while  violation of \RA{9} requires two computations for the refinement operation Eq. \ref{eq:rel-refinement2}, namely $C_{i,j} \diamond C_(j,k)$ and $(C_{j,k}\breve{~} \diamond C_{i,j}\breve{~})\breve{~}$ (lines \ref{alg:ra9start}--\ref{alg:ra9end}). \RA{4} and \RA{10} are not used by the algorithm, since
this would complicate the algorithm unduly.

\section{A Quantitative Account of Qualitative Calculi}
In this section, we report on computational properties of specific calculi
which are beyond the computational complexity of constraint-based reasoning.
For example, one might be interested to know how many relations are typically sufficient to describe a scene of $n$ objects unequivocally or with a specific residual uncertainty.
To this end, we developed two empirical measures that characterize certain aspects of qualitative calculi that are arguably relevant to applications.
We want to answer two questions: (1) How well do calculi with many relations make use of the usually higher information content?
(2) Does information content differ significantly between the six classes of calculi established in Section \ref{sec:properties}?


The first measure we consider is {\em information content} of the composition operation.
Our motivation is to estimate how much additional information can be gained by applying a composition operation.
This allows us to estimate whether, for example, having observed relations $r(A,B)$ and $r'(B,C)$ in a scene, it is worthwhile to observe $r''(A,C)$ too,
as it may be improbable to derive $r''$ by composition ($r\circ r'$).
To obtain more general results we consider sequences of compositions $r\circ r' \circ r'' \circ \ldots$ for several lengths.
We define the information content $I$ of a relation $R \subseteq \Rel$ to be 
\begin{equation}
I(R)=1-\frac{|R|}{|\Rel|}
\end{equation}

\noindent
where $|\Rel|$ denotes the number of base relations of the calculus,
and $|R|$ the number of base relations $R$ consists of.
In case of the universal relation this results in $I(U)=0$ as nothing is known, 
$I(r)=1-\frac{1}{|\Rel|}$ for all base relations $r\in\Rel$, 
and $I(\emptyset)=1$.
Obviously, the more base relations a calculus involves, the higher the information content {\em can be} for base relations. 
Therefore, we define 
\begin{equation}
I_C^k = \frac{\sum_{R\in r^k} I(R)}{|\Rel|^{k+1}}
\end{equation}
for a calculus $C$  with $r^k=\{r^0\circ\ldots\circ r^k|r^0,\ldots,r^k \in \Rel\}$
to be the average information content after $k$ composition operations, 
i.e., how restrictive relation are on average after information propagation with composition.
In particular, $I_C^k$ is 1 minus the average proportion of base relations in any cell in the composition table.
For example, for QTC$_{C22}$ ($|\Rel|=209$) or the  Cardinal Direction Relations (CDR) ($|\Rel |=218$) $I^0\approx 1$, 
whereas for the Point Calculus with three base relations $I_{PC}^0\approx 0.67$.
%
%
We apply an iterative method to derive the values of $I_C^k$ that constructs $r^k$ for $k=0,1,\ldots$ rather than looping across combinations of base relations.
Despite the potentially exponential size of $r^k$, the calculation remains feasible in many cases. Only for OPRA$_m$ with $m\geq 3$ and some QTC variants we were not able to derive values for higher $k$ in reasonable time. For the other calculi, computation was terminated after 14 compositions or if $I_C^k$ drops below $0.5$.


As a second measure we determine the average degree of overlap that occurs after $k$ steps of composition for selected calculi.
The degree of overlapping $O(R_i,R_j)$ is determined by counting the number of atomic relations shared by two relations, normalized by the total number of base relations:
\begin{equation}
O(R_i,R_j) = \frac{|R_i\cap R_j|}{|\Rel|}
\end{equation}
For example, if two relations in a calculus with eight base relations share four base relations,
the overlap is $0.5$.
This value indicates how the information content differs between dealing with base relations only versus dealing with arbitrary relations (and thus how the results on information content generalize to arbitrary relations).
Similar to $I(R)$ and $I_C^k$,
 we define $O_C^k$ to be the average overlap over all composition chains of length k.
\begin{equation}
O_C^k = \frac{\sum_{R_i,R_j\in r^k} O(R_i,R_j)}{|\Rel|^{k+1}}
\end{equation}

The results of the two measures are summarized in Figure \ref{fig:information-content} 
and Table \ref{tab:informationContent}, 
showing information content versus length $k$ of composition chains.

\begin{figure}
\addtolength{\abovecaptionskip}{-16pt}
\addtolength{\belowcaptionskip}{6pt}
\centerline{%
\includegraphics[width=0.81\textwidth]{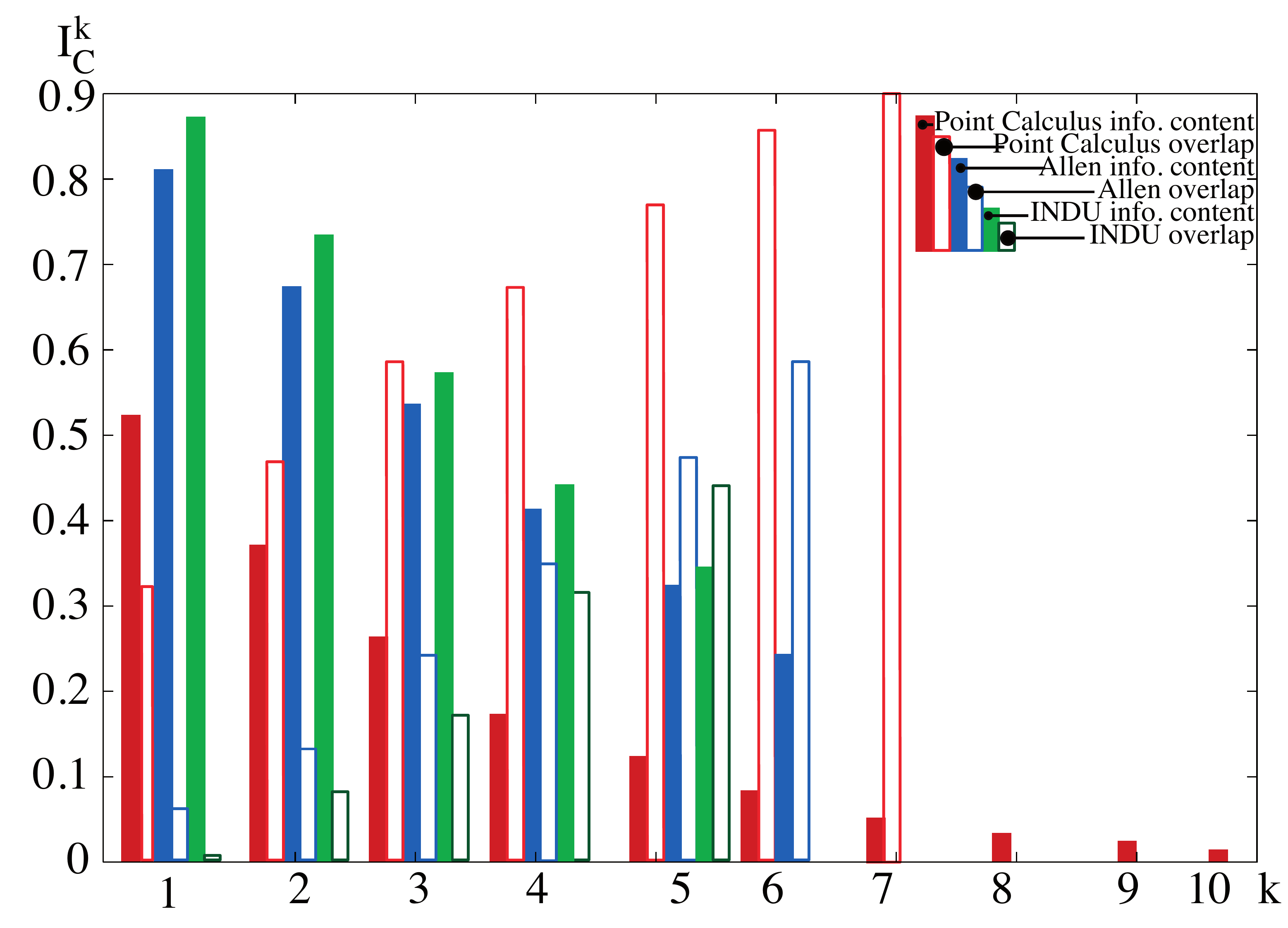}
}
\caption{Information content and overlap after $k$ compositions for selected calculi}
\label{fig:information-content}
\end{figure}


\begin{table}[tb!]
  \centering
  \begin{small}
    \rowcolors{1}{lightblue}{}%
    \begin{tabular}{l*{14}{r}r}
      \hline\rowcolor{medblue}\rule{0pt}{9pt}%
      Calculus
		& 0 & 1 & 2 & 3 & 4 & 5 & 6 & 7 & 8 & 9 & 10 & 11 & 12  & 13 & 14\\
      \hline\rule{0pt}{9pt}%
      Allen 
	& 92.3
	& 81.4 & 66.8 & 52.8 & 41.1 & 31.8 & 24.5 & 18.9 
	& 14.5 & 11.2 &   8.6 &   6.6 &   5.1 &    3.9 & 3.0 \\
     Block Algebra 
	& 99.4
	& 96.5 & 89.0 & 77.7 & 65.3 & 53.4 & 43.0 & 34.1 & 27.0 & 21.1 & 16.4 & 12.8 & 9.9 & 7.7 & 5.9 \\
     CDC 
	& 88.9
	& 76.8 & 60.4 & 44.5 & 31.6 & 21.9 & 14.9 & 10.1 & 6.8 & 4.6 & 3.1 & 2.0 & 1.4 & 0.9 & 0.6  \\
     CYC$_{b}$
	& 75.0
	& 62.5 & 46.9 & 32.8 & 21.9 & 14.1 & 8.8 & 5.4 & 3.2 & 1.9 & 1.1 & 0.6 & 0.4 & & \\ 
     DRA$_{fp}$ 
	& 98.8
	& 89.9 & 69.0 & 45.0 & 25.8 & 13.4 & 6.5 & 3.0 & 1.3 & 0.6 & 0.2 &  &  & \\ 
     DRA-con 
	& 85.7
	& 74.6 & 59.0 & 43.4 & 30.4 & 20.5 & 13.5 & 8.7 & 5.6 & 3.5 & 2.2 & 1.3 & 0.8 & 0.5 & 0.3 \\
     Point Calculus
	& 66.7
	& 51.9 & 37.0 & 25.5 & 17.3 & 11.6 & 7.8 & 5.2 & 3.5 & 2.3 & 1.5 & 1.0 & 0.7 & 0.5\\
     RCC-5
	& 80.0
	& 56.8 & 34.9 & 19.7 & 10.6 & 5.5 & 2.7 & 1.3 & 0.6 & 0.3 &  &  &  & & \\
     RCC-8
	& 87.5
	& 62.3 & 38.0 & 21.1 & 11.0 & 5.5 & 2.6 & 1.2 & 0.6 & 0.3 &  &  &  & \\
    STAR$_4$
	& 88.9
	& 66.9 & 45.0 & 28.5 & 17.4 & 10.3 & 6.0 & 3.5 & 2.0 & 1.1 & 0.6 & 0.4 &  & & \\
     \hline\rule{0pt}{9pt}%
     DRA$_{f}$
	& 98.6
	& 90.6 & 70.4 & 46.3 & 26.7 & 13.9 & 6.7 & 3.0 & 1.3 & 0.6 & 0.2 &  &  & & \\
     INDU
	& 96.0
	& 86.9 & 72.5 & 57.5 & 44.1 & 33.2 & 24.7 & 18.2 & 13.4 & 9.9 & 7.2 & 5.3 & 4.0 & 2.9 & 2.1\\
     OPRA$_{1}$ 
	& 95.0
	& 82.0 & 55.8 & 30.8 & 14.5 & 6.2 & 2.4 & 0.9 & 0.3 &  &  &  &  & & \\
     OPRA$_{2}$
	& 98.6
	& 90.3 & 64.1 & 32.9 & 13.0 & 4.3 & 1.3 & 0.3 &  &  &  &  &  & & \\
     OPRA$_{3}$
	& 99.4
	& 93.1 & 71.4 & 40.2 & 16.7 & 5.6 & & & & & & & & &\\
     OPRA$_{4}$
	& 99.6
	& 94.6 & 76.7 & 48.0 & & & & & & & & & & &\\
     \hline\rule{0pt}{9pt}%
     QTC$_{B11}$
	& 88.9
	& 90.0 & 93.2 &95.8 & 97.5 & 98.6 &99.1 & 99.5 & 99.7 & & & & & & \\
     QTC$_{B12}$
	& 94.1
	& 91.2 & 90.5 & 91.3 & 92.8 & 94.2 & 95.6 & & & & & & & & \\
     QTC$_{B21}$
	& 88.9
	& 0.0 & 0.0 & 0.0 & 0.0& 0.0& 0.0& 0.0& 0.0& 0.0& 0.0& 0.0& 0.0& 0.0 &  0.0 \\
     QTC$_{B22}$
	& 96.3
	& 51.9 & 37.0 & 25.5 &17.3 & 11.6 & 7.8 & 5.2 & 3.5 &2.3 &1.5 & 1.0 & 0.7 & 0.5 & 0.3 \\
     QTC$_{C21}$
	& 98.8
	& 92.5 & 76.6 & 68.6 & 69.5 & 73.0 & 76.5 & 79.4 & 81.8 & 83.7 & 85.2 & 86.4 &	87.4 &&\\
     QTC$_{C22}$ 
	& 99.5
	& 95.1 & 78.0 & 69.3 & 51.2 &&&&&&&&&&\\
     \hline\rule{0pt}{9pt}%
     RDR
	& 97.2
	& 82.6 & 63.2 & 45.7 & 32.0 & 22.0 & 15.0 & 10.1 & 6.8 & 4.6 & 3.1 & 2.0 & 1.4 & 0.9 & 0.6 \\
     CDR
	& 99.5
	& 78.8 & 60.9 & 48.9 & 39.6 & 32.1 & 26.1 & 21.2 & 17.2 & 14.0 & 11.4 & 9.3 & 7.6 & 6.2 & 5.1 \\
      \hline
    \end{tabular}
  \end{small}
  \par\smallskip
\caption{Information content $I_C^k$ for calculi  in \%
\label{tab:informationContent}
}
\end{table}

Figure \ref{fig:information-content} shows that 
the average information content for the Point Calculus after $1$ step is $\approx 0.52$
and additionally, the overlap of $\approx 0.33$ is already quite high after a single composition.
Therefore, in order to obtain detailed information it is reasonable to also observe $r_{AC}$ between objects $A$ and $C$ even if $r_{AB}$ and $r_{BC}$ are already known.
By contrast, the INDU calculus has a very high information content ($\approx 0.87$) and a much smaller overlap. Therefore, it is not so informative to observe $r_{AC}$ as a lot of information is preserved after a composition.
It is clear that the $O_{C}^k$ grows for increasing $k$ as composition results become coarser step by step.
Nevertheless, information loss for PC is much higher than compared to Allen and INDU calculus: $I_{INDU}^5$ and $I_{Allen}^5$ are close to $I_{PC}^2$ ($O_{C}^k$ respectively).

Our results show that there is no evidence for a relation between the information content of a calculus and its classification as per Figure \ref{fig:algebra_notions}. The only exceptions are some of the QTC calculi as $I_C^k$ starts to increase after some $k$ with increasing $k$.

Although the calculi start with quite different values for $I^0$, most calculi have an information content less than $0.1$ after six steps. The most notable exception is the Block Algebra where $I_{BA}^6\approx 0.43$ and even after ten compositions it remains above $16\%$. Only Allen, INDU and CDR are somehow comparable. 
Concerning the classes we derived in Section \ref{sec:properties} no uniform behavior can be observed. 
Thus, from a perspective of expressive power of calculi, there is no argument against working with calculi that are not relation algebras.
We have to note that the comparison of the values for calculi where it is known that a-closure decides consistency and those where it does not (or is unknown) is problematic.
The latter ones may contain relations which are not physically realizable and thus reduce the value of information content.

%

There are some interesting observations wrt.\ the various QTC variants.
The QTC$_{B1x}$, QTC$_{B21}$ and QTC$_{C21}$ calculi behave differently from other calculi, whereas QTC$_{B22}$ behaves `normally', i.e., $I_{}^k$ increases, although it is very closely related to the other QTC variants.
Interestingly, QTC$_{B1x}$ and QTC$_{C21}$ are the only calculi where $I_{}^k$ increases with growing $k$. From our perspective, the reason lies in the multimodal structure of the calculus. 
As it combines points with line segments, the composition table (CT) contains empty relations, since an object cannot be interpreted as a point and a line segment at the same time.
Additionally, the CT contains only fairly small relations, i.e., with small $|R|$.
For example, the CT of QTC$_{B12}$ contains $29\%$ empty relations, $29\%$ atomic relations, and $42\%$ other relations which have a maximal size of $|R|\leq 3$.
The results for QTC$_{B21}$ are not surprising as the composition table only contains the universal relation and thus, for all $k$, $I^k=0.0$ and $O^k=100.0$.
For QTC$_{C21}$ we observe that $I^k$ decreases to $69.5\%$ at step 4, but starts to increase for $k\geq 5$. We assume that this is also the case for QTC$_{C22}$, as it is a refinement of QTC$_{C21}$, but we were not able to calculate necessary values due to the high complexity. So far, we have no explanation for this decrease.

An additional observation is that PC and QTC$_{B22}$ are similar with respect to information content, i.e., $I_{PC}^k\approx I_{QTC_{B22}}^k$ for $k\leq 14$.
This congruence is interesting as the overlap values vary, the underlying partition scheme is different and the difference in base relations is significant (three for PC vs. 27 for $QTC_{B22}$).
We leave the question of connections between these two calculi for future research.

\section{Conclusion}
We have looked at spatio-temporal representation and reasoning from an algebraic perspective,
examining the implications of algebraic properties on modeling and reasoning algorithms,
and testing these properties for a representative corpus of existing calculi.
The resulting classification shows that calculi which have been described early in the literature
tend to reside in the upper part of Figure \ref{fig:algebra_notions}; that is, they tend to have a rich algebraic structure.
Few more recently developed calculi are based on generalizations and have a weaker structure.
We have been able to conclude that common reasoning procedures are incorrect for the latter class of calculi,
and have proposed a corrected universal a-closure algorithm that makes use of reasoning optimizations where they are allowed.
Furthermore, we found that algebraic properties do not necessarily relate to how much information is preserved in successive reasoning steps.

An interesting and significant line of future work is to extend this study to ternary calculi,
which requires an extension of binary relation algebras to ternary relations, see also \cite{Sci00}.

\paragraph{Acknowledgments:}  We would like to thank Immo Colonius, Arne Kreutzmann, Jae Hee Lee, Andr\'e Scholz and Jasper van de Ven for inspiring discussions during the ``spatial reasoning tea time''.  
This work has been supported by the DFG-funded SFB/TR~8 ``Spatial Cognition'', projects \mbox{R3-[QShape]} and \mbox{R4-[LogoSpace]}.

\linespread{0.965}
\bibliographystyle{plain}
\bibliography{../calculi+algebra}

\clearpage
\appendix
\section{Additional proofs from Section \ref{sec:requirements}}
\label{app:proofs4Req}

\subsection{Proof of Fact \ref{fact:weak+strong_conv+comp_general}}
Given a qualitative calculus $(\Rel,\Int,\breve{},\diamond)$
and relations $R,S \subseteq \Rel$,
the following hold:
\begin{align}
  \varphi(R\breve{~}) & \supseteq \varphi(R)\breve{~}         \label{eq:abstract_converse_general_app}   \\
  \varphi(R \diamond S)  & \supseteq \varphi(R) \circ \varphi(S) \label{eq:abstract_composition_general_app}
\end{align}
    If $C$ has
    weak converse, then, for all $R \subseteq \Rel$:
    \begin{equation}
      R\breve{~} = \bigcap\{S \subseteq \Rel \mid \varphi(S) \supseteq \varphi(R)\breve{~}\} 
      \label{eq:weak_converse_general_app}
    \end{equation}
    If $C$ has
    strong converse, then, for all $R \subseteq \Rel$:
    \begin{equation}
      \varphi(R\breve{~}) = \varphi(R)\breve{~} 
      \label{eq:strong_converse_general_app}
    \end{equation}
    If $C$ has
    weak composition, then, for all $R, S \subseteq \Rel$:
    \begin{equation}
      R\diamond S = \bigcap\{T \subseteq \Rel \mid \varphi(T) \supseteq \varphi(R)\circ\varphi(S)\}
      \label{eq:weak_composition_general_app}
    \end{equation}
    If $C$ has
    strong composition, then, for all $R, S \subseteq \Rel$:
    \begin{equation}
      \varphi(R \diamond S) = \varphi(R) \circ \varphi(S)
      \label{eq:strong_composition_general_app}
    \end{equation}

\begin{proof}
  For \eqref{eq:abstract_converse_general_app}, consider
  \begin{xalignat*}{2}
    \varphi(R\breve{~}) & = \bigcup_{r \in R} \varphi(r\breve{~})              & & \text{~definition of~}\varphi(R\breve{~}) \\
                        & \supseteq \bigcup_{r \in R} \varphi(r)\breve{~}      & & \text{~property \eqref{eq:abstract_converse}} \\
                        & = \left(\bigcup_{r \in R} \varphi(r)\right)\breve{~} & & \text{~distributivity in set theory}      \\
                        & = \varphi(R)\breve{~}                                & & \text{~definition of~} \varphi(R).
  \intertext{For \eqref{eq:abstract_composition_general_app}, consider}
    \varphi(R\diamond S) & = \bigcup_{r\in R}\bigcup_{s\in S} \varphi(r\diamond s)              & & \text{~definition of~}\varphi(R\diamond S)       \\
                         & \supseteq \bigcup_{r\in R}\bigcup_{s\in S} \varphi(r)\circ\varphi(s) & & \text{~property \eqref{eq:abstract_composition}} \\
                         & = \left(\bigcup_{r\in R}\varphi(r)\right) \circ \left(\bigcup_{s\in S}\varphi(s)\right) 
                                                                                                & &  \text{~distributivity in set theory}            \\
                         & = \varphi(R) \circ \varphi(S)                                        & & \text{~definition of~} \varphi(R) \text{~and~} \varphi(S)
  \end{xalignat*}
  Properties \eqref{eq:strong_converse_general_app} and \eqref{eq:strong_composition_general_app}
  are proven using \eqref{eq:strong_converse} and \eqref{eq:strong_composition}
  in the same way as we have just proven \eqref{eq:abstract_converse_general_app} and \eqref{eq:abstract_composition_general_app}
  using \eqref{eq:abstract_converse} and \eqref{eq:abstract_composition}.
  
  For \eqref{eq:weak_converse_general_app},
  let $R = \{r_1, \dots, r_n\}$ for some $n \geqslant 1$ and $r_1, \dots, r_n \in \Rel$.
  Due to Definition \ref{def:stronger_versions_of_comp+conv} \eqref{eq:weak_converse},
  we have that
  \[
    r_i\breve{~} = \bigcap\{S \subseteq \Rel \mid \varphi(S) \supseteq \varphi(r_i)\breve{~}\} 
  \]
  for every $i=1,\dots,n$.
  Let $S_{i1}, \dots, S_{im_i}$ be the $S$ over which the above intersection ranges, i.e.,
  \[
    r_i\breve{~} = \bigcap_{j=1}^{m_i}  S_{ij}.
  \]
  Due to Definition \ref{def:qualitative_calculus},
  we have that
  \[
    R\breve{~} ~~=~~ \bigcup_{i=1}^{n} r_i\breve{~}
               ~~=~~ \bigcup_{i=1}^{n} \bigcap_{j=1}^{m_i}  S_{ij} \\
               ~~=~~ \bigcap_{j_1=1}^{m_1} \dots \bigcap_{j_n=1}^{m_n} \bigcup_{i=1}^{n} S_{ij_i}\,,
  \]
  where the last equality is due to the distributivity of intersection over union.
  Now \eqref{eq:weak_converse_general_app} follows if we show
  that, for every $S \in \Rel$, the following are equivalent.
  \begin{enumerate}
    \item
      $\varphi(S) \supseteq \varphi(R)\breve{~}$
    \item
      there exist $S_1,\dots,S_n$ with $S = S_1 \cup\dots\cup S_n$ and $\varphi(S_i) \supseteq \varphi(r_i)\breve{~}$ for every $i=1,\dots,n$.
  \end{enumerate}
  For ``$1 \Rightarrow 2$'', assume $\varphi(S) \supseteq \varphi(R)\breve{~}$,
  i.e., $\varphi(S) \supseteq \bigcup_{i=1}^{n}\varphi(r_i)\breve{~}$ (Definition \ref{def:qualitative_calculus}).
  If we further assume that $S = \{s_1, \dots, s_\ell\}$, which implies that 
  $\varphi(S) \supseteq \bigcup_{j=1}^{\ell}\varphi(s_j)$ (Definition \ref{def:qualitative_calculus}),
  then we can choose $S_i = \{s_j \mid \varphi(s_j) \cup \varphi(r_i)\breve{~} = \emptyset\}$
  for every $i=1,\dots,n$. Because $C$ is based on JEPD relations,
  we have that $\varphi(S_i) \supseteq \varphi(r_i)\breve{~}$.
  
  For ``$2\Rightarrow1$'', let $S = S_1 \cup\dots\cup S_n$ and $\varphi(S_i) \supseteq \varphi(r_i)\breve{~}$ for every $i=1,\dots,n$.
  Due to Definition \ref{def:qualitative_calculus} and because $C$ is based on JEPD relations,
  we have that $\varphi(S) = \bigcup_{i=1}^{n}\varphi(S_i)$.
  Hence, $\varphi(S) \supseteq \bigcup_{i=1}^{n} \varphi(r_i)\breve{~}$ due to the assumption,
  and $\varphi(S) \supseteq \varphi(R)\breve{~}$ due to Definition \ref{def:qualitative_calculus}.

  \eqref{eq:weak_composition_general_app} is proven analogously.
  \qed
\end{proof}

\subsection[RA5 and RA5l from Table \ref{tab:relation_algebra_axioms} are equivalent given RA1, RA4 and RA6--RA9]{\RA5 and \RA{5l} from Table \ref{tab:relation_algebra_axioms} are equivalent given \RA1, \RA4 and \RA6--\RA9}
We only show that \RA5 implies \RA{5l}; the converse direction is analogous.
\begin{xalignat*}{2}
  r \diamond (s \cup t) & = (r\breve{~})\breve{~} \diamond ((s \cup t)\breve{~})\breve{~}                     & & (\RA7) \\
                        & = ((s \cup t)\breve{~} \diamond r\breve{~})\breve{~}                                & & (\RA9) \\
                        & = ((t\breve{~} \cup s\breve{~}) \diamond r\breve{~})\breve{~}                       & & (\RA8) \\
                        & = ((t\breve{~} \diamond r\breve{~}) \cup (s\breve{~} \diamond r\breve{~}))\breve{~} & & (\RA5) \\
                        & = ((r \diamond t)\breve{~} \cup (r \diamond s)\breve{~})\breve{~}                   & & (\RA9) \\
                        & = ((r \diamond t)\breve{~})\breve{~} \cup ((r \diamond s)\breve{~})\breve{~}        & & (\RA8) \\
                        & = (r \diamond t) \cup (r \diamond s)                                                & & (\RA7) \\
                        & = (r \diamond s) \cup (r \diamond t)                                                & & (\RA1)
\end{xalignat*}
\qed

\subsection[RA6 and RA6l from Table \ref{tab:relation_algebra_axioms} are equivalent given RA7 and RA9]{\RA6 and \RA{6l} from Table \ref{tab:relation_algebra_axioms} are equivalent given \RA7 and \RA9}
We only show that \RA6 implies \RA{6l}; the converse direction is analogous.
We first establish that $\id\breve{~} = \id$.
\begin{xalignat*}{2}
  \id\breve{~} & = \id\breve{~} \diamond \id                     & & (\RA6) \\
               & = \id\breve{~} \diamond (\id\breve{~})\breve{~} & & (\RA7) \\
               & = (\id\breve{~} \diamond \id)\breve{~}          & & (\RA9) \\
               & = (\id\breve{~})\breve{~}                       & & (\RA6) \\
               & = \id                                           & & (\RA7)
\intertext{Now we use this lemma to establish \RA{6l}.}
  \id \diamond r  & = (\id\breve{~})\breve{~} \diamond (r\breve{~})\breve{~} & & (\RA7)  \\
                  & = (r\breve{~} \diamond \id\breve{~})\breve{~}            & & (\RA9)  \\
                  & = (r\breve{~} \diamond \id)\breve{~}                     & & (\text{Lemma}) \\
                  & = (r\breve{~})\breve{~}                                  & & (\RA6)  \\
                  & = r                                                      & & (\RA7)
\end{xalignat*}
\qed

\section{Additional proofs from Section \ref{sec:relation_algebras}} 
\label{app:proofs4RA}

\subsection{Proof of Fact \ref{fact:minimal_axiom_set_implied_by_calculus_def}}
Axioms \RA1--\RA3 are always satisfied because they are a characterization of a Boolean algebra;
and the set operations on the relations form a Boolean algebra because
$\varphi$ maps base relations to a set of JEPD relations and complex relations
to sets of interpretations of base relations.

The definition of the converse and composition operations for non-base relations
in Definition \ref{def:qualitative_calculus} ensures that Axioms \RA5 and \RA8 hold.

Axiom \RA[sup]7 always holds due to JEPD and weak converse:
For every $r \in \Rel$, we have that
\[
  \varphi(r\breve{~}\breve{~}) \supseteq \varphi(r\breve{~})\breve{~} \supseteq \varphi(r)\breve{~}\breve{~} = \varphi(r),
\]
where the first inclusion is due to Fact \ref{fact:weak+strong_conv+comp_general} \eqref{eq:abstract_converse_general}
with $R = r\breve{~}$,
the second inclusion is due to Definition \ref{def:qualitative_calculus} \eqref{eq:abstract_converse} for $r$,
and the equation is due to the properties of binary relations over the universe $\Univ$.
Since the $\varphi(r)$ are a set of JEPD relations, $r\breve{~}\breve{~} \supseteq r$ follows.
This reasoning carries over to arbitrary relations.

Axioms \WA[sup] and \SA[sup] always hold due to JEPD and weak composition:
For every $r \in \Rel$, we have that
\[
  \varphi((r \diamond 1) \diamond 1) \supseteq \varphi(r \diamond 1) \circ \varphi(1)
                                     =         \varphi(r \diamond 1) \circ (\Univ \times \Univ)
                                     \supseteq \varphi(r \diamond 1),
\]
where the first inclusion is due to to Fact \ref{fact:weak+strong_conv+comp_general} \eqref{eq:abstract_composition_general}
with $R = r \diamond 1$ and $S = 1$,
and the last inclusion is due to the fact that $R \circ (\Univ \times \Univ) \supseteq R$ for any binary relation $R \subseteq \Univ \times \Univ$.
Since the $\varphi(r)$ are a set of JEPD relations, $(r \diamond 1) \diamond 1 \supseteq r \diamond 1$  follows.
Again, this reasoning carries over to arbitrary relations.

Axioms \RA[sub]6, \RA[sub]{6l}, \RA[sub]7 
are violated by the following calculus.
Let $\Rel = \{r_1,r_2\}$, $\Univ = \{0,1\}$, $\id = r_1$, $1 = \{r_1,r_2\}$ with:
\begin{xalignat*}{3}
  \varphi(r_1)  & = \{(0,0),(0,1)\} & r_1\breve{~}  & = 1 & r_1 \diamond r_1 & = 1   \\
  \varphi(r_2)  & = \{(1,0),(1,1)\} & r_2\breve{~}  & = 1 & r_1 \diamond r_2 & = r_1 \\
                &                   &               &     & r_2 \diamond r_1 & = 1   \\
                &                   &               &     & r_2 \diamond r_2 & = r_2
\end{xalignat*}
This calculus satisfies the conditions from Definition \ref{def:qualitative_calculus}
but violates Axioms \RA[sub]6, \RA[sub]{6l}, \RA[sub]7: 
\begin{xalignat*}{2}
  & \RA[sub]6    & r_1 \diamond \id = 1         & ~\nsubseteq~ r_1 \\
  & \RA[sub]{6l} & \id \diamond r_1 = 1         & ~\nsubseteq~ r_1 \\
  & \RA[sub]7    & r_1\breve{~}\breve{~} = 1    & ~\nsubseteq~ r_1
\end{xalignat*}

Axioms \WA[sub], \SA[sub], \RA[sub]4, \RA[sup]4, \RA[sup]6, \RA[sup]{6l}, \RA[sub]9, \RA[sup]9, \RA[sub]{10}, \RA[sup]{10}, \PL[right], \PL[left] are violated by the following calculus.
Let $\Rel = \{r_1,r_2,r_3,r_4\}$, $\Univ = \{0,1\}$, $\id = r_1$, $1 = \{r_1,r_2\}$ with:
\par\noindent
\parbox{.45\textwidth}{%
  \begin{xalignat*}{3}
    \varphi(r_1)  & = \{(0,0)\} & r_1\breve{~}  & = r_1 \\
    \varphi(r_2)  & = \{(1,1)\} & r_2\breve{~}  & = r_2 \\
    \varphi(r_3)  & = \{(0,1)\} & r_3\breve{~}  & = r_4 \\
    \varphi(r_4)  & = \{(1,0)\} & r_4\breve{~}  & = r_3
  \end{xalignat*}
}%
\hfill
\parbox{.55\textwidth}{%
  \begin{align*}
    \begin{array}{l|llll}
      \quad \text{right operand} & r_1       & r_2       & r_3       & r_4       \\
      \text{left operand}\quad \diamond &    &           &           &           \\
      \hline
      \qquad r_1                 & r_1       & \emptyset & r_3       & \emptyset \\
      \qquad r_2                 & \emptyset & r_3       & \emptyset & r_4       \\
      \qquad r_3                 & \emptyset & r_3       & \emptyset & r_1,r_4   \\
      \qquad r_4                 & r_1,r_4   & \emptyset & r_2       & \emptyset
    \end{array}
  \end{align*}
}
\par\noindent
This calculus satisfies the conditions from Definition \ref{def:qualitative_calculus}
but violates Axioms \WA[sub], \SA[sub], \RA[sub]4, \RA[sup]4, \RA[sup]6, \RA[sup]{6l}, \RA[sub]9, \RA[sup]9, \RA[sub]{10}, \RA[sup]{10}, \PL[right], \PL[left]:
\begin{small}
\begin{align*}
  & \WA[sub], \SA[sub]:~~         (r_1 \diamond 1) \diamond 1 = 1                                                                     ~\nsubseteq~    \{r_1,r_3,r_4\} = r_1 \diamond 1                                                 \\
  & \RA[sub]4:~~                  (r_1 \diamond r_3) \diamond r_4 = r_3 \diamond r_4 = \{r_1,r_4\}                                    ~\nsubseteq~    r_1 = r_1 \diamond \{r_1,r_4\} = r_1 \diamond (r_3 \diamond r_4)                 \\
  & \RA[sup]4:~~                  (r_4 \diamond r_3) \diamond r_4 = r_2 \diamond r_4 = r_4                                            ~\nsupseteq~    \{r_1,r_4\} = r_4 \diamond \{r_1,r_4\} = r_4 \diamond (r_3 \diamond r_4)         \\
  & \RA[sup]6:~~                  r_2 \diamond \id = r_2 \diamond r_1 = \emptyset                                                     ~\nsupseteq~    r_2                                                                              \\
  & \RA[sup]{6l}:~~               \id \diamond r_2 = r_1 \diamond r_2 = \emptyset                                                     ~\nsupseteq~    r_2                                                                              \\
  & \RA[sub]9, \RA[sup]9:~~       (r_3 \diamond r_4)\breve{~} = \{r_1,r_4\}\breve{~} = \{r_1,r_3\}                                    ~\nsubeqsubset~ \{r_1,r_4\} = r_3 \diamond r_4 = r_4\breve{~} \diamond r_3\breve{~}              \\
  & \RA[sub]{10}, \RA[sup]{10}:~~ r_3\breve{~} \diamond \overline{r_3\diamond r_1} = r_4 \diamond \emptyset = r_4 \diamond 1 = \{r_1,r_2,r_4\} ~\nsubeqsubset~ \{r_2,r_3,r_4\} = \overline{r_1}                                        \\
  & \PL[right]:~~                 (r_1 \diamond r_4) \cap r_1\breve{~} = \emptyset \cap r_1 = \emptyset                               ~~\text{but}~~  (r_4 \diamond r_1) \cap r_1\breve{~} = \{r_4,r_1\} \cap r_1 = r_4 \neq \emptyset \\
  & \PL[left]:~~                  (r_4 \diamond r_1) \cap r_1\breve{~} = \{r_4,r_1\} \cap r_1 = r_1 \neq \emptyset                    ~~\text{but}~~  (r_1 \diamond r_1) \cap r_4\breve{~} = r_1 \cap r_3 = \emptyset
\end{align*}
\end{small}
\qed

\begin{remark}
  Of course, there are calculi that satisfy only the weak conditions from Definition \ref{def:qualitative_calculus}
  but are a relation algebra, for example the following.
  Let $\Rel = \{r_0,r_1\}$, $\Univ = \{0,1\}$, $\id = r_1$, $1 = \{r_1,r_2\}$ with:
  \begin{xalignat*}{3}
    \varphi(r_1)  & = \{(0,0),(0,1)\} & r_1\breve{~}  & = r_2 & r_1 \diamond r_1 & = r_1 \\
    \varphi(r_2)  & = \{(1,0),(1,1)\} & r_2\breve{~}  & = r_1 & r_1 \diamond r_2 & = 1   \\
                  &                   &               &       & r_2 \diamond r_1 & = 1   \\
                  &                   &               &       & r_2 \diamond r_2 & = r_2
  \end{xalignat*}
\end{remark}

\section{Detailed description of the test results in Section \ref{sec:properties} by calculus}
\label{app:descrOfProperties}

\begin{description}
  \item[9-Intersection]
    ~\par
    \begin{itemize}
      \item
        \Hets: all axioms are satisfied.
      \item
        \SparQ: all 6 tests are passed.
    \end{itemize}
  
%
  \item[Allen's interval relations]
    ~\par
    \begin{itemize}
      \item
        \cite{Mos07}: this calculus is a relation algebra.
      \item
        \Hets: all axioms are satisfied.
      \item 
        \SparQ: all 6 tests are passed.
    \end{itemize}

  \item[Block Algebra]
    ~\par
    \begin{itemize}
      \item
        \Hets: all axioms are satisfied. 
      \item 
        \SparQ: all 6 tests are passed.
    \end{itemize}

  \item[Cardinal Direction \emph{Calculus}]
    ~\par
    \begin{itemize}
      \item
        \cite{Mos07}: this calculus is a relation algebra, as already reported in \cite{ligozat-JVLC:98}.
      \item
        \Hets: all axioms are satisfied.
      \item 
        \SparQ: all 6 tests are passed.
    \end{itemize}

  \item[Cardinal Direction \emph{Relations}]
    ~\par
    \begin{itemize}
      \item
        \Hets: 
          \RA6 is violated for all base relations but one,
          \RA{6l} for only 209 base relations,
          \RA7 for 214 base relations,
          \RA9 for 5,607 pairs, 
          \RA{10} for 41,834 pairs,
          \PL for 22,976 triples,
          \RA4 for 2,936,946 triples,
          and
          \SA for 38 base relations.
          \WA is satisfied.
          CDC is therefore just a Boolean algebra with distributivity, weak associativity and weak involution.
      \item 
        \SparQ: the version with 218 base relations fails all 6 tests:
        \RA6 and \RA{6l} for 217 base relations, \RA7 for 214 base relations,
        \RA9 for 45,939 pairs (97\%),\footnote{This number is too large because SparQ uses a variant of \RA9.}
        \PL for 22,976 triples ($<$\,1\%),
        and \RA4 for 2,936,946 triples (28\%).
    \end{itemize}

    \pagebreak[2]
    Counterexamples:
    \par
    \parbox{\linewidth}{%
      \begin{footnotesize}
        First 5 counterexamples for \PL
        \begin{align*}
          \nonpeirceext{\Reln{W-NW-N-NE-E}}{\Reln{NW-N-NE}}{\Reln{B-S}}{}{\Reln{B}}   \displaybreak[0]\\[4pt]
          \nonpeirceext{\Reln{W-NW-N-NE-E}}{\Reln{NW-N-NE}}{\Reln{B}}{}{\Reln{B}}     \displaybreak[0]\\[4pt]
          \nonpeirceext{\Reln{W-NW-N-NE-E}}{\Reln{NW-N}}{\Reln{S-E-SE}}{}{\Reln{B}}   \displaybreak[0]\\[4pt]
          \nonpeirceext{\Reln{W-NW-N-NE-E}}{\Reln{NW-N}}{\Reln{B-S-SE}}{}{\Reln{B}}   \displaybreak[0]\\[4pt]
          \nonpeirceext{\Reln{W-NW-N-NE-E}}{\Reln{NW-N}}{\Reln{B-S-E-SE}}{}{\Reln{B}}
        \end{align*}
        \par
      \end{footnotesize}%
    }%
    \par
    \parbox{.5\linewidth}{%
      \begin{footnotesize}
        First 10 counterexamples for \RA9
        \begin{align*}
          \noninvdist{\Reln{S}}{\Reln{S}} \\
          \noninvdist{\Reln{S}}{\Reln{SW}} \\
          \noninvdist{\Reln{S}}{\Reln{W}} \\
          \noninvdist{\Reln{S}}{\Reln{NW}} \\
          \noninvdist{\Reln{S}}{\Reln{N}} \\
          \noninvdist{\Reln{S}}{\Reln{NE}} \\
          \noninvdist{\Reln{S}}{\Reln{E}} \\
          \noninvdist{\Reln{S}}{\Reln{SE}} \\
          \noninvdist{\Reln{S}}{\Reln{B}} \\
          \noninvdist{\Reln{S}}{\Reln{S-SW}}
        \end{align*}
        \par
      \end{footnotesize}%
    }%
    \parbox{.5\linewidth}{%
      \begin{footnotesize}
        First 10 counterexamples for \RA{10}
        \begin{align*}
          \nontarskidm{\Reln{S}}{\Reln{S}} \\
          \nontarskidm{\Reln{S}}{\Reln{SW}} \\
          \nontarskidm{\Reln{S}}{\Reln{W}} \\
          \nontarskidm{\Reln{S}}{\Reln{NW}} \\
          \nontarskidm{\Reln{S}}{\Reln{N}} \\
          \nontarskidm{\Reln{S}}{\Reln{NE}} \\
          \nontarskidm{\Reln{S}}{\Reln{E}} \\
          \nontarskidm{\Reln{S}}{\Reln{SE}} \\
          \nontarskidm{\Reln{S}}{\Reln{B}} \\
          \nontarskidm{\Reln{S}}{\Reln{S-SW}}
        \end{align*}
        \par
      \end{footnotesize}%
    }%
    \par
    \parbox{.6\linewidth}{%
      \begin{footnotesize}
        First 10 counterexamples for \SA
        \begin{align*}
          \nonsemiassoc{\Reln{B-S-W-NW}} \\
          \nonsemiassoc{\Reln{B-W-N-NE}} \\
          \nonsemiassoc{\Reln{B-N-E-SE}} \\
          \nonsemiassoc{\Reln{B-S-SW-E}} \\
          \nonsemiassoc{\Reln{B-S-NE-E}} \\
          \nonsemiassoc{\Reln{B-S-W-SE}} \\
          \nonsemiassoc{\Reln{B-SW-W-N}} \\
          \nonsemiassoc{\Reln{B-NW-N-E}} \\
          \nonsemiassoc{\Reln{B-S-W-NW-SE}} \\
          \nonsemiassoc{\Reln{B-SW-W-N-NE}}
        \end{align*}
        \par
      \end{footnotesize}%
    }%
    \par
      \begin{footnotesize}
        First 5 counterexamples for \RA4
        \begin{align*}
          \Nonassoc{\Reln{W-NW-N-NE-E-SE}}{\Reln{W-NW-N-NE-E-SE}}{\Reln{W-NW-N-NE-E}} \displaybreak[0]\\[4pt]
          \Nonassoc{\Reln{W-NW-N-NE-E-SE}}{\Reln{W-NW-N-NE-E-SE}}{\Reln{W-NW-N-NE}}   \displaybreak[0]\\[4pt]
          \Nonassoc{\Reln{W-NW-N-NE-E-SE}}{\Reln{W-NW-N-NE-E-SE}}{\Reln{W-NW-N}}      \displaybreak[0]\\[4pt]
          \Nonassoc{\Reln{W-NW-N-NE-E-SE}}{\Reln{W-NW-N-NE-E-SE}}{\Reln{S-SW-W-SE}}   \displaybreak[0]\\[4pt]
          \Nonassoc{\Reln{W-NW-N-NE-E-SE}}{\Reln{W-NW-N-NE-E-SE}}{\Reln{S-SW-W-E-SE}}
        \end{align*}
        \par
      \end{footnotesize}%

  \item[CYC${}_\text{b}$]
    ~\par
    \begin{itemize}
      \item
        \Hets: all axioms are satisfied.
      \item
        \SparQ: all 6 tests are passed.
    \end{itemize}

  \item[Dependency Calculus]
    ~\par
    \begin{itemize}
      \item
        \cite{Mos07}: this calculus is a relation algebra (and homomorphically embeddable into RCC-5).
      \item 
        \Hets: all axioms are satisfied.
      \item 
        \SparQ: All 5 tests are passed.
    \end{itemize}

  \item[Dipole Calculus]
    ~\par
    \begin{itemize}
      \item
        \Hets:
        and DRA${}_{\text{f}}$ violates \RA4 for 71,424 triples of base relations.
        For example, we have that
        \[
            \nonassoc[]{\Reln{rrrl}}{\Reln{rrrl}}{\Reln{llrl}}.
        \]
        However, DRA${}_{\text{f}}$ satisfies the weaker \axiom{WA} and \axiom{SA}.
        The violation of associativity attributed in \cite{MoratzEtAl2011} to the converse operation being weak,
        illustrated by the example
        $\Reln{bfii} \diamond \Reln{lllb} = \Reln{llll}$ for DRA${}_{\text{f}}$.
        DRA${}_{\text{fp}}$ and DRA-connectivity satisfy all axioms.
      \item
        SparQ: \RA4 fails for 71,424 triples in DRA${}_{\text{f}}$ (about 19\%). The other 4 tests are passed.
        DRA${}_{\text{fp}}$ and DRA-connectivity pass all 6 tests.
    \end{itemize}

  \item[Geometric Orientation]
    ~\par
    \begin{itemize}
      \item
        \Hets: all axioms are satisfied.
      \item
        \SparQ: all 6 tests are passed.
    \end{itemize}

  \item[INDU]
    ~\par
    \begin{itemize}
      \item 
        \Hets: 1,880 triples violate \RA4; 
        the violation of associativity has already been reported and attributed to the absence of strong composition
        in \cite{BCL06}: for example,
        \[
            \nonassoc[]{\Reln{bi}^>}{\Reln{mi}^>}{\Reln{m}^>}.
        \]
        All other axioms are satisfied,
        including \axiom{SA} and \axiom{WA}. Therefore, the Interval-Duration Calculus
        is a semi-associative relation algebra.
      \item 
        \SparQ: All tests except \RA4 are passed; \RA4 fails on 1880 triples (12 \%).
    \end{itemize}

  \item[{\boldmath $\text{OPRA}_n$}]
    ~\par
    \begin{itemize}
      \item 
        \cite{Mos07}: for $n=1,2,3,4$, either the calculus is not associative, or there is an error in the composition table.
      \item 
        \Hets: \RA4 is violated by 
        \begin{tabular}[t]{rl}
                  1,664 & triples for OPRA${}_1$, \\
                257,024 & triples for OPRA${}_2$, \\
              2,963,952 & triples for OPRA${}_3$, \\
             16,711,680 & triples for OPRA${}_4$, \\
             63,840,000 & triples for OPRA${}_5$, \\
            190,771,200 & triples for OPRA${}_6$, \\
            481,275,648 & triples for OPRA${}_7$, \\
          1,072,693,248 & triples for OPRA${}_8$.
        \end{tabular}
        
        All other axioms are satisfied, including \axiom{SA} and \axiom{WA}. 
        Therefore, $\text{OPRA}_n$, $n \leqslant 8$,
        are semi-associative relation algebras.
        %
        %
        %
      \item 
        \SparQ: All tests except \RA4 are passed.
        In OPRA${}_1$, 1664 triples (21\%) do not agree with \RA4;
        in OPRA${}_2$, there are 257,024 violations (69\%);
        in OPRA${}_3$   2,963,952 (78\%);
        in OPRA${}_4$  16,711,680 (83\%);
        in OPRA${}_5$  63,840,000 (86\%);
        and
        in OPRA${}_6$ 190,771,200 (88\%).
    \end{itemize}    

    \parbox{\linewidth}{%
      \begin{footnotesize}
        First 10 counterexamples
        for \RA4, $\text{OPRA}_1$:
        \begin{align*}
          \nonassoc{3_{3}}{3_{2}}{x},\qquad x = 0_3,\dots,0_0 \\
          \nonassoc{3_{3}}{3_{0}}{x},\qquad x = 2_3,\ldots,2_0 \\
          \nonassoc{3_{3}}{2_{3}}{x},\qquad x = 3_2,3_0
        \end{align*}
        \par
      \end{footnotesize}%
    }
    \parbox{\linewidth}{%
      \begin{footnotesize}
        First 10 counterexamples
        for \RA4, $\text{OPRA}_2$:
        \[
          \nonassoc[]{7_{7}}{7_{7}}{x},\qquad x = 6_7,\dots,6_0,5_7,5_6
        \]
        \par
      \end{footnotesize}%
    }
    \parbox{\linewidth}{%
      \begin{footnotesize}
        First 10 counterexamples
        for \RA4, $\text{OPRA}_3$:
        \begin{align*}
          \nonassoc{11_{11}}{11_{11}}{11_{i}},\qquad i=10,8,6,4,2,0 \\
          \nonassoc{11_{11}}{11_{11}}{10_{i}},\qquad i=11,\dots,8
        \end{align*}
        \par
      \end{footnotesize}%
    }
    \parbox{\linewidth}{%
      \begin{footnotesize}
        First 10 counterexamples
        for \RA4, $\text{OPRA}_n$, $n=4,\dots,8$, with $m=4n-1$:
        \[
          \nonassoc[]{m_{m}}{m_{m}}{m_{i}},\qquad i = m,\dots,m-9
        \]
        \par
      \end{footnotesize}%
    }%

  \item[Point Calculus]
    ~\par
    \begin{itemize}
      \item
        \Hets: all axioms are satisfied.
      \item 
        \SparQ: all 6 tests are passed.
    \end{itemize}

  \item[QTC]
    ~\par
    \begin{itemize}
      \item
        \cite{Mos07}:
        \id is not a neutral element of composition in QTC$_{\text{B11}}$ and QTC$_{\text{B12}}$.
        After introducing a new \id relation and making the relations JEPD, these two calculi pass all tests.
      \item
        \Hets: QTC$_{\text{B11}}$ and QTC$_{\text{B12}}$ violate \RA6 and \RA{6l}
        for all base relations that are not \id.        
        QTC$_{\text{B21}}$ (9 base relations) and QTC$_{\text{B22}}$ (27 base relations)
        violate \RA6 and \RA{6l} for \emph{all} base relations.
        \par
        QTC$_{\text{C21}}$ (81 base relations) does not satisfy \RA6 and \RA{6l} for all base relations but one,
        \RA4 for 292,424 triples, \RA9 for 160 pairs, \RA{10} for 80 pairs, and \PL for 1056 triples.\footnote{%
          \label{ftn:special_treatment_R10}%
          Note that, for calculi that violate \RA9,
          the equivalence between \PL and \RA{10} is no longer ensured,
          hence the mentioning of both of them.
          Furthermore, \RA{10} is the only axiom that should be tested for all relations,
          but we have only tested it for all base relations.
          Therefore, there could be more violations than the four listed.
          The same cautions apply to QTC$_{\text{C22}}$.%
        }
        QTC$_{\text{C22}}$ (209 base relations) does not satisfy \RA6 and \RA{6l} for 208 relations,
        \RA9 for 1248 pairs, \RA{10} for 624 pairs, \PL for 12,768 triples, and \RA4 for 7,201,800 triples,
        see also footnote \ref{ftn:special_treatment_R10}.
      \item
        \SparQ: QTC$_{\text{B11}}$ and QTC$_{\text{B12}}$
        are reported to have no identity relation,
        and they both fail the test for \RA6 and \RA{6l} with all base relations but one.
        \par
        QTC$_{\text{B21}}$ and QTC$_{\text{B22}}$
        are also reported to have no identity relation
        and fail the test for \RA6 and \RA{6l} with \emph{all} base relations.
        \par
        QTC$_{\text{C21}}$ (81 base relations) fails the test for 
        \RA4 (292,424 triples, 55\%),
        \RA6 and \RA{6l} (80 out of 81 relations),
        \RA9 (160 pairs, 2\%), and
        \PL (1056 triples, $<$1\%).
        \par
        QTC$_{\text{C22}}$ (209 base relations) shows a similar result: it fails the test for
        \RA4 (7,201,800 triples, 79\%),
        \RA6 and \RA{6l} (208 out of 209 relations),
        \RA9 (1248 pairs, 3\%), and
        \PL (12,768 triples, $<$1\%).
    \end{itemize}

  \item[RCC]
    ~\par
    \begin{itemize}
      \item
        \cite{Mos07}: after certain repairs to the composition table, RCC-5 and RCC-8 are found to be relation algebras.
      \item
        \Hets: all axioms are satisfied.
      \item
        \SparQ: after one repair (removing \textsf{EC} from the entry $\textsf{NTPPi} \diamond \textsf{NTPP}$), RCC-5 and RCC-8 pass all 6 tests.
    \end{itemize}

  \item[Rectangular Direction Relations]
    ~\par
    \begin{itemize}
      \item
        \Hets: 
          \RA6 is violated for all base relations but one,
          \RA{6l} for only 33 base relations,
          \RA7 for 32 base relations,
          \RA9 for 855 pairs, 
          \RA{10} for 671 pairs,
          and 
          \PL for 3424 triples.
          \RA4, \SA, and \WA are satisfied.
          CDC is therefore just a Boolean algebra with distributivity, associativity and weak involution.
      \item 
        \SparQ: the calculus fails the following tests:
        \RA6 and \RA{6l} for 35 out of 36 base relations, \RA7 for 32 base relations,
        \RA9 for 855 pairs (66\%),
        and
        \PL for 3424 triples (7\%).
    \end{itemize}

    \pagebreak[4]
    Counterexamples:
    \par
    \parbox{\linewidth}{%
      \begin{footnotesize}
        First 10 counterexamples for \PL
        \begin{align*}
          \nonpeirce{\Reln{B}}{\Reln{N}}{\Reln{B:W}} \\
          \nonpeirce{\Reln{B}}{\Reln{N}}{\Reln{B}} \\
          \nonpeirce{\Reln{B}}{\Reln{S}}{\Reln{B:W:NW:N:NE:E}} \\
          \nonpeirce{\Reln{B}}{\Reln{S}}{\Reln{B:N:NE:E}} \\
          \nonpeirce{\Reln{B}}{\Reln{S}}{\Reln{B:W:NW:N}} \\
          \nonpeirce{\Reln{B}}{\Reln{S}}{\Reln{B:W:E}} \\
          \nonpeirce{\Reln{B}}{\Reln{S}}{\Reln{B:N}} \\
          \nonpeirce{\Reln{B}}{\Reln{S}}{\Reln{B:E}} \\
          \nonpeirce{\Reln{B}}{\Reln{S}}{\Reln{B:W}} \\
          \nonpeirce{\Reln{B}}{\Reln{S}}{\Reln{B}}
        \end{align*}
        \par
      \end{footnotesize}%
    }%
    \par
    \parbox{.5\linewidth}{%
      \begin{footnotesize}
        First 10 counterexamples for \RA9
        \begin{align*}
          \noninvdist{\Reln{B}}{\Reln{S:SW}} \\
          \noninvdist{\Reln{B}}{\Reln{SW}} \\
          \noninvdist{\Reln{B}}{\Reln{SE}} \\
          \noninvdist{\Reln{B}}{\Reln{NW}} \\
          \noninvdist{\Reln{B}}{\Reln{NE}} \\
          \noninvdist{\Reln{B}}{\Reln{W}} \\
          \noninvdist{\Reln{B}}{\Reln{E}} \\
          \noninvdist{\Reln{B}}{\Reln{N}} \\
          \noninvdist{\Reln{B}}{\Reln{S}} \\
          \noninvdist{\Reln{B}}{\Reln{B}}
        \end{align*}
        \par
      \end{footnotesize}%
    }%
    \parbox{.5\linewidth}{%
      \begin{footnotesize}
        First 10 counterexamples for \RA{10}
        \begin{align*}
          \nontarskidm{\Reln{B}}{\Reln{S:SW}} \\
          \nontarskidm{\Reln{B}}{\Reln{SW}} \\
          \nontarskidm{\Reln{B}}{\Reln{SE}} \\
          \nontarskidm{\Reln{B}}{\Reln{NW}} \\
          \nontarskidm{\Reln{B}}{\Reln{NE}} \\
          \nontarskidm{\Reln{B}}{\Reln{W}} \\
          \nontarskidm{\Reln{B}}{\Reln{E}} \\
          \nontarskidm{\Reln{B}}{\Reln{N}} \\
          \nontarskidm{\Reln{B}}{\Reln{S}} \\
          \nontarskidm{\Reln{B}}{\Reln{B}}
        \end{align*}
        \par
      \end{footnotesize}%
    }%

  \item[STAR{\boldmath${}_4$}]
    ~\par
    \begin{itemize}
      \item
        \Hets: all axioms are satisfied.
      \item 
        \SparQ: all 6 tests are passed.
    \end{itemize}
\end{description}

\end{document}